\documentclass[11pt]{article}
\usepackage[a4paper]{geometry}
\usepackage{amsthm}
\usepackage{amsmath}
\usepackage{amssymb}
\usepackage{amsfonts}
\usepackage{graphicx}
\usepackage{color}
\usepackage[bf,SL,BF]{subfigure}
\usepackage{url}
\usepackage{epstopdf}
\usepackage{pdfsync}
\usepackage[colorlinks]{hyperref}
\usepackage[all]{xypic}
\usepackage{comment}

\usepackage{upgreek}

\hypersetup{
    linkcolor=blue,
}

\newlength{\fixboxwidth}
\usepackage{epsfig,amsbsy,graphicx,multirow}

\usepackage{ algorithm, algorithmic}
\renewcommand{\algorithmiccomment}[1]{\bgroup\hfill//~#1\egroup}

\usepackage[all]{xy}
\usepackage{accents}

 \numberwithin{equation}{section}

\def\R{\mathbb{R}}

\def\cN{\mathcal{N}}

\def\Y{{\bf\mathcal{Y}}}

\def\E{\mathbb{E}}

\def\B{\mathcal{B}}

\def\X{{\bf\mathcal{X}}}

\def\I{\mathcal{I}}

\def\L{\mathcal{L}}

\def\cW{\mathcal{W}}

\def\H{\mathcal{H}}

\def\<{\big\langle}
\def\>{\big\rangle}
\def\Img{\operatorname{Im}}
\def\Ker{\operatorname{Ker}}

\def\diiv{\operatorname{div}}

\def\Tr{\operatorname{Tr}}

\def\Span{\operatorname{span}}

\def\s{s}

\definecolor{red}{rgb}{0.9, 0, 0}

\newtheorem{Theorem}{Theorem}[section]
\newtheorem{Proposition}[Theorem]{Proposition}

\newtheorem{Corollary}[Theorem]{Corollary}
\newtheorem{Remark}[Theorem]{Remark}

\newtheorem{Problem}{Problem}

\begin{document}
\title{Kernel Flows:\\
from learning kernels from data into the abyss}
%\\ to  bottomless networks without guesswork }
%\title{ From learning kernels from data\\to learning very deep networks }

\date{\today}

\author{Houman Owhadi\thanks{Corresponding author. California Institute of Technology, 1200 E California Blvd, MC 9-49, Pasadena, CA 91125, USA, owhadi@caltech.edu} and Gene Ryan Yoo\thanks{California Institute of Technology, 1200 E California Blvd, MC 253-47, Pasadena, CA 91125, USA, gyoo@caltech.edu} }

\maketitle

\begin{abstract}
Learning can be seen as approximating an unknown function by interpolating the training data.
Kriging offers a solution to this problem based on the prior specification of a kernel.
We explore a numerical approximation approach  to  kernel selection/construction based on the simple premise that a kernel must be good if the number of interpolation points  can be halved without significant loss in accuracy (measured using the intrinsic RKHS norm $\|\cdot\|$ associated with the kernel).
We first test and motivate this idea on a simple problem of recovering the Green's function of an elliptic PDE (with inhomogeneous coefficients) from the sparse observation of one of its solutions. Next we consider the problem of learning non-parametric families of deep kernels of the form $K_1(F_n(x),F_n(x'))$ with $F_{n+1}=(I_d+\epsilon G_{n+1})\circ F_n$ and $G_{n+1} \in \Span\{K_1(F_n(x_i),\cdot)\}$.
With the proposed approach constructing the kernel becomes equivalent to integrating a  stochastic data driven dynamical system, which allows for the training of very deep (bottomless) networks and the exploration of their properties. These networks learn by constructing flow maps in the kernel and input spaces via  incremental data-dependent deformations/perturbations  (appearing as the cooperative counterpart of adversarial examples) and, at profound depths, they (1) can achieve accurate classification from only one data point per class (2) appear to learn archetypes of each class (3) expand distances between points that are in different classes and contract distances between points in the same class.

For  kernels parameterized by the weights of  Convolutional Neural Networks, minimizing  approximation errors incurred by halving  random subsets of interpolation points, appears  to outperform training (the same CNN architecture)  with relative entropy and dropout.
 \end{abstract}

\section{Introduction}
Despite their popularity and impressive achievements \cite{lecun2015deep} Artificial Neural Networks (ANNs) remain difficult to analyze. From a deep kernel learning perspective \cite{wilson2016deep}, the action of the last layer of an ANN can be seen as that of regressing the data with a kernel parameterized by the weights of all the previous layers. Therefore analysing the problem of performing a regression of the data with a kernel that is also learnt from the data could help understand ANNs and elaborate a rigorous theory for deep learning.
Hierarchical Bayesian Inference \cite{schwaighofer2005learning} (placing a prior on a space of kernels and conditioning on the data) and Maximum Likelihood Estimation \cite{williams1996gaussian} (choosing the kernel which maximizes the probability of observing the data) are well known approaches for learning the kernel.
In this paper we explore a numerical approximation approach (motivated by interplays between Gaussian Process Regressions and Numerical Homogenization \cite{OwhScobook2018})  based on the simple premise that a kernel must be good if the number of points $N$ used to perform the interpolation of data can be reduced to $N/2$ without significant loss in accuracy (measured using the intrinsic RKHS norm $\|\cdot\|$ associated with the kernel). Writing $u$ and $v$ for the interpolation of the data with $N$ and $N/2$ points, the relative error $\rho=\frac{\|u-v\|^2}{\|u\|^2}$ induces a data dependent ordering on the space of kernels.
The Fr\'{e}chet derivative of $\rho$   identifies the direction of the  gradient descent and leads to a simple  algorithm (Kernel Flow) for its minimization:
(1) Select $N_f$ ($\leq N$) points (at random, uniformly, without replacement) from the $N$ training data points  (2)
Select $N_c=N_f/2$ points (at random, uniformly, without replacement) from the $N_f$ points (3) Perturb the kernel in the  gradient descent direction of  $\rho$ (computed from the current kernel and the $N_f, N_c$ points) (4) Repeat.

To provide some context for this algorithm, we first summarize (in Section \ref{sec2}) interplays between Kriging, Gaussian Process Regression, Game Theory and Optimal Recovery. The identification (in Sec.~\ref{sec3}) of
$\rho$ and its  Fr\'{e}chet derivative  leads (in  Sec.~\ref{secfamker}) to the proposed algorithm in a parametric setting.

In Sec.~\ref{secjhdkhj33} we describe interplays between the proposed algorithm and Numerical Homogenization \cite{OwhScobook2018}
by implementing and testing the parametric version of Kernel Flow for the (simple and amenable to analysis) problem of (1) recovering the unknown conductivity $a$ of the PDE $-\diiv(a\nabla u)=f$ based on seeing $u\in \H^1_0((0,1))$ at a finite number $N$ of points (2) approximating $u$ between measurement points. For this problem, $a$, $u$ and $f$ are all unknown, we only know that $f\in L^2$ and we try to learn the Green's function of the PDE (seen as a kernel parameterized by the unknown conductivity $a$). Experiments suggest that, by minimizing $\rho$ (parameterized by the conductivity), the algorithm can recover the conductivity and significantly improve the accuracy of the interpolation.

\begin{figure}[h]
\begin{center}
\includegraphics[width= \textwidth]{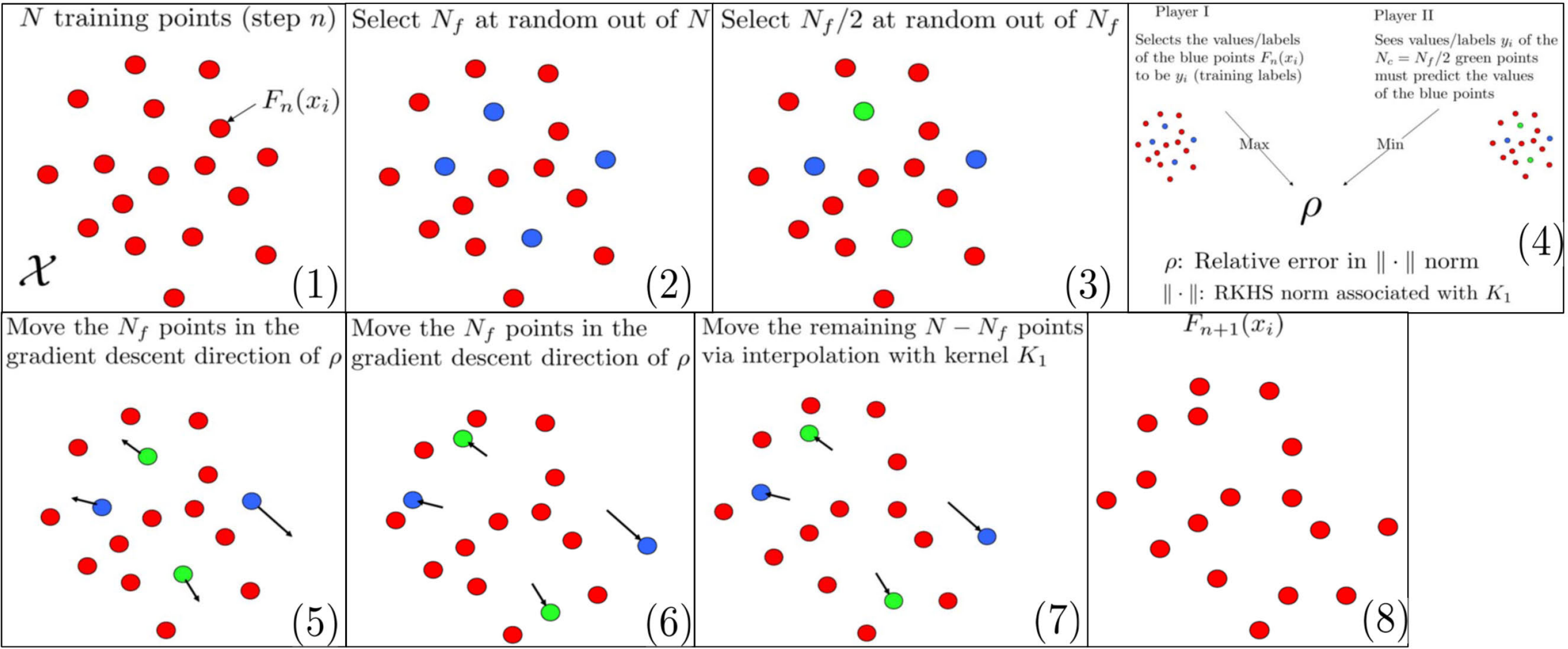}
\caption{The game theoretic interpretation of the step $n\rightarrow n+1$ of Kernel Flow. (1) Starting from $F_n$ and
 the $N$ data points $(F_n(x_i),y_i)$ (2)  Select $N_f$ indices out of $N$ (3) Select $N_f/2$ indices  out of $N_f$ (4) Consider the zero sum adversarial game where Player I chooses the labels of the $N_f$ points to be $y_i$ and Player II sees half of them and tries to guess the other and let
 $\rho$ be the loss of Player II in that game (using relative error in the RKHS norm associated with $K_1$) (5, 6) Move the $N_f$ selected points $F_n(x_i)$ to decrease the loss of Player II (7) Move the remaining $N-N_f$ (and any other point $x$) points via interpolation with the kernel $K_1$ (this specifies  $F_{n+1}$) (8) Repeat.}
\label{figkf2}
\end{center}
\end{figure}

Next (in Sec.~\ref{secKernelFlow}) we derive a non parametric version of the proposed algorithm that learns
a kernel of the form
\begin{equation}
K_n(x,x')=K_{1}(F_n(x),F_n(x'))\,,
\end{equation}
where $K_1$ is a standard kernel (e.g. Gaussian $K_1(x,x')=e^{-\gamma |x-x'|^2}$) and
$F_n$ maps the input space into itself, $n\rightarrow F_n$ is a discrete flow in the input space,
and $F_{n+1}$ is obtained from $F_n$ by interrogating random subsets of  the training data as described in Fig.~\ref{figkf2} (which also summarizes the game theoretic interpretation of the proposed Algorithm, note that the game is incrementally rigged to minimize the loss of Player II).

\begin{figure}[h]
\begin{center}
\includegraphics[width= \textwidth]{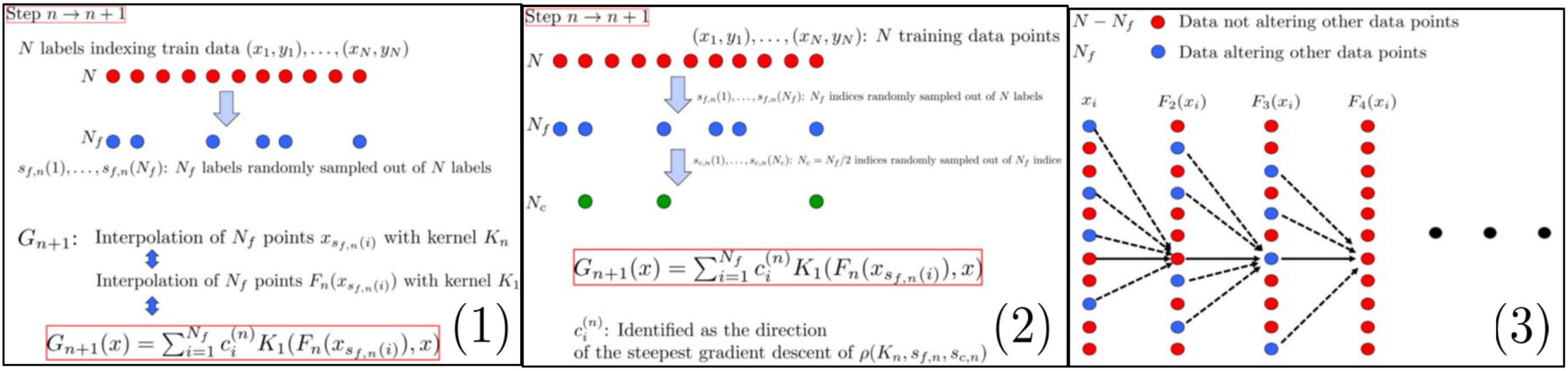}
\caption{The Kernel Flow Algorithm. (1) $N_f$ indices $s_{f,n}(1),\ldots,s_{f,n}(N_f)$ are randomly sampled out of $N$ and $G_{n+1}$ belongs to the linear span of the $K(F(x_{s_{f,n}(i)}^{(n)}),x)$ (2) the coefficients in the representation of $G_{n+1}$ are found as the direction of the  gradient descent of $\rho$ (3) The value of $F_{n+1}(x_i)$ is the sum of $F_n(x_i)$ a small perturbation depending on the joint values of the
$F_n(x_{s_{f,n}(j)})$.
}
\label{figkf1}
\end{center}
\end{figure}

\begin{figure}[h]
\begin{center}
\includegraphics[width= \textwidth]{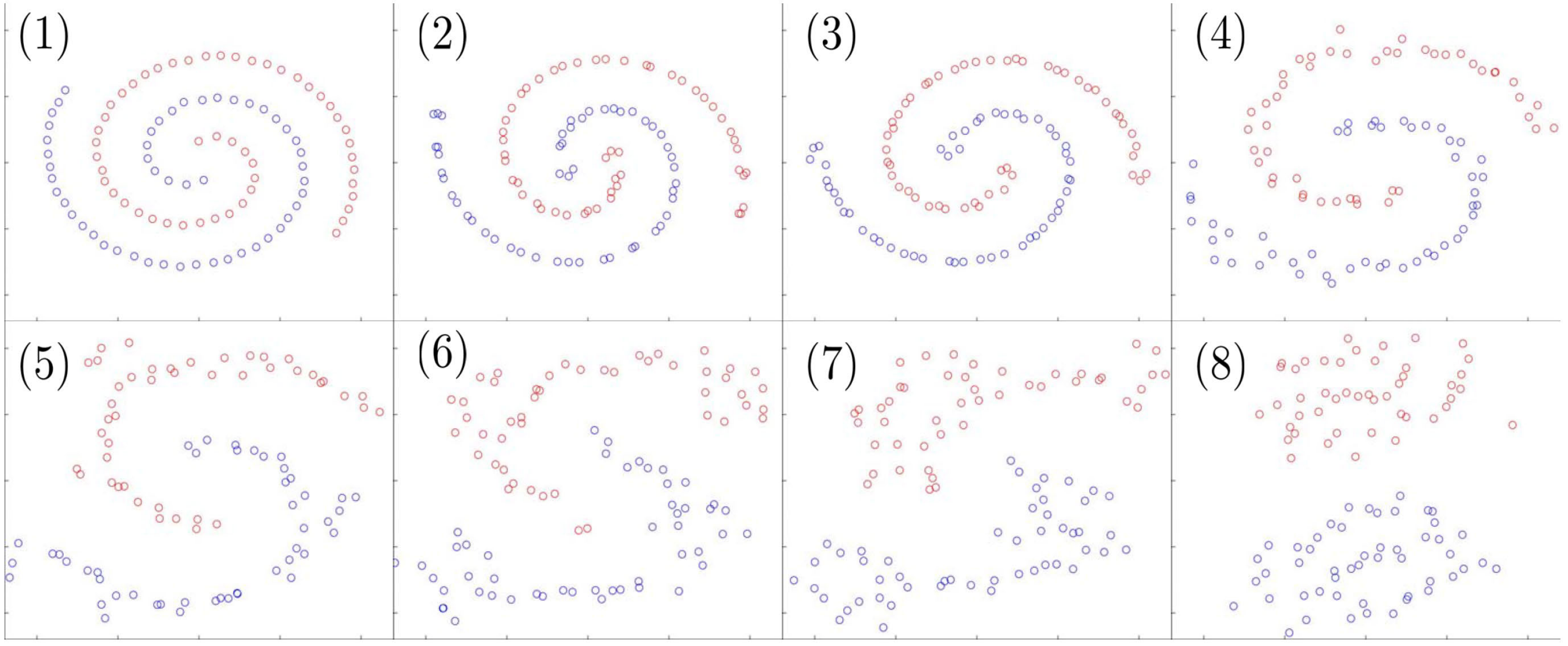}
\caption{Swiss Roll Cheesecake. $N=100$.  Red points have label $-1$ and blue points have label $1$. $F_n(x_i)$ for 8 different values of $n$. }
\label{figswissroll2}
\end{center}
\end{figure}
\begin{figure}[h]
\begin{center}
\includegraphics[width= \textwidth]{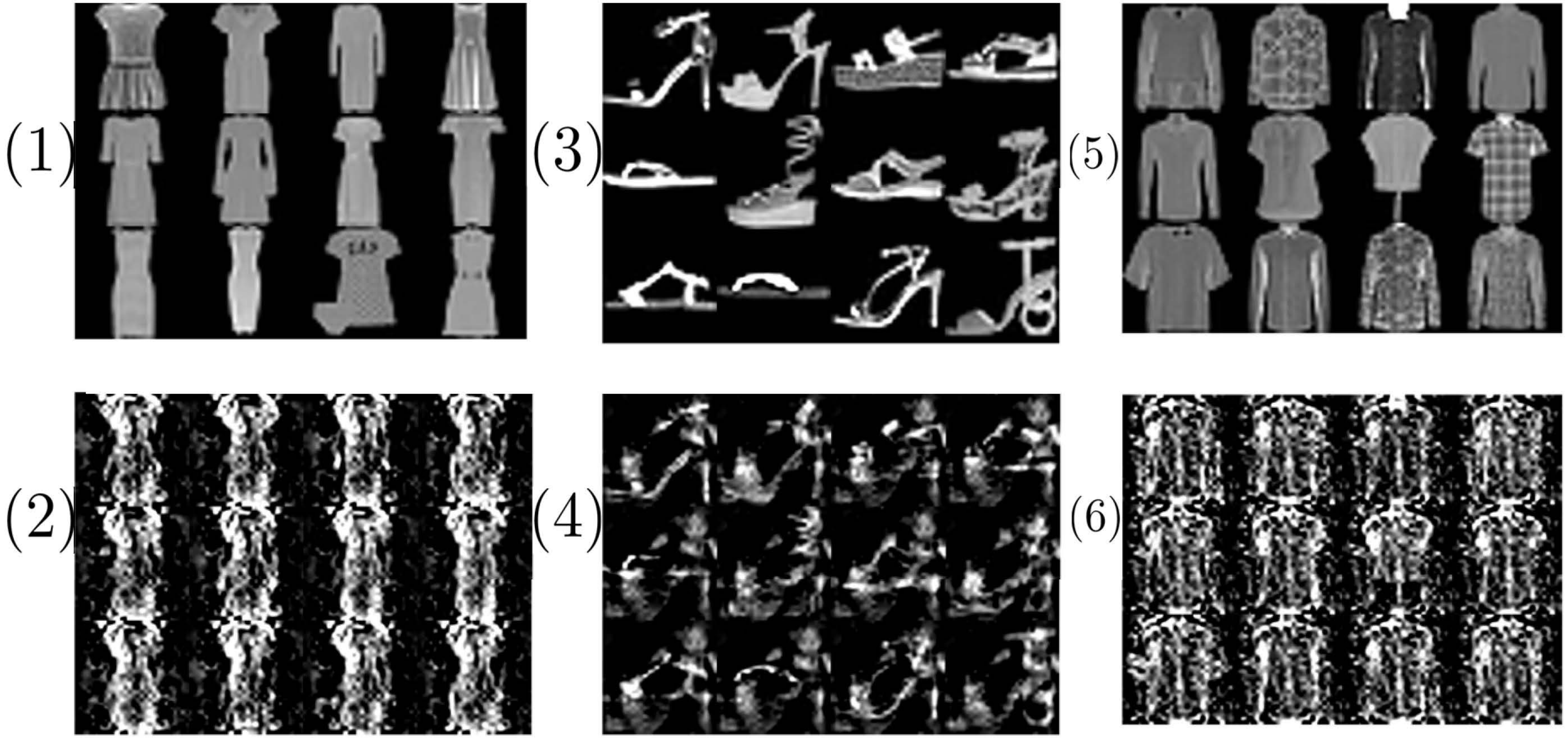}
\caption{Results for Fashion-MNIST. $N=60000$, $N_f=600$ and $N_c=300$. (1, 3, 5) Training data $x_i$ (class 3, 5 and 6) (2, 4, 6) $F_n(x_i)$ (class 3, 4 and 6) for $n=50000$. }
\label{figfmnist2}
\end{center}
\end{figure}

The proposed algorithm (see Fig.~\ref{figkf1} for a  summary of its  structure) can be reduced to an iteration of the form
 \begin{equation}
K_n(x,x')=K_{n-1}(x+\epsilon G_n(x), x+ \epsilon G_n(x')).
\end{equation}
 Writing $x_i$ for the training points and $F_1(x)=x$ and $F_n(x)=(I_d+\epsilon G_n)\circ F_{n-1}(x)$, the network $F_n$ (composed of $n$ layers) is learnt from the data in a recursive manner (across layers) by
(1) using $g_i^{(n)}:=G_n \circ F_{n-1}(x_{s_{f,n}(i)})$ for a random subset $\{x_{s_{f,n}(i)}|1\leq i \leq N_f\}$  of the points $\{x_i|1\leq i\leq N\}$ as training parameters and interpolating $G_n$ with the kernel $K_1$ in between the points $F_{n-1}(x_{s_{f,n}(i)})$ (2) selecting $g_i^{(n)}$ in the direction of the  gradient descent of $\rho$ at each step.

Writing $(x_i,y_i)$ for the $N$ training data points, $G_{n+1}$ ends up being  of the form
\begin{equation}\label{eqkejdkdh}
G_{n+1}(x)=\sum_{i=1}^{N_f} c_i^{(n)} K_1(F_n(x_{s_{f,n}(i)}),x)\,
\end{equation}
where $s_{f,n}(1),\ldots, s_{f,n}(N_f)$ are $N_f$ indices sampled at random (uniformly without replacement) from $\{1,\ldots,N\}$.
Note that  the kernel  $K_n$ produced at any step $n$, is a Deep Hierarchical Kernel (in the sense of \cite{wilson2016deep, steinwart2016learning})
satisfying the  nesting equations
\begin{equation}
K_n(x,x')=K_{n-1}(x+\epsilon \sum_{i=1}^{N_f} c_i^{(n-1)} K_{n-1}(x_{s_{f,n-1}(i)},x), x'+\epsilon \sum_{i=1}^{N_f} c_i^{(n-1)} K_{n-1}(x_{s_{f,n-1}(i)},x'))\,.
\end{equation}
Furthermore, the structure of the network defined by that kernel is randomized through the random selection of the $N_f$ points (see Fig.~\ref{figkf1}.3).
The coefficients $c_i^{(n)}$ in \eqref{eqkejdkdh} are identified through one step of gradient descent of the  relative error $\rho$ (measured in the RKHS norm associated with $K_1$) of the approximation of the labels (the $y_{s_{f,n}(i)}$) of those $N_f$ points upon seeing half of them. From the game theoretic perspective of Fig.~\ref{figkf2}, $\rho$ is the loss of Player II (attempting the guess the unseen labels)
 and the points $F_{n+1}(x_{s_{f,n}(i)})$ are perturbations of the points $F_n(x_{s_{f,n}(i)})$ in a direction which seeks to minimize this loss.
 Note also that the proposed (Kernel Flow) algorithm produces a  flow $F_n$ (randomized through sampling of the training data) in the input space and a (stochastic) dynamical system $K_1(F_n(x),F_n(y))$ in the kernel space.
Since learning becomes  equivalent to integrating a dynamical system, it does not require back-propagation nor guessing the architecture of the network, which enables the construction of very deep networks and the exploration of their properties.

We implement this algorithm (and visualize its flow) for MNIST \cite{yann1998mnist}, Fashion-MNIST \cite{xiao2017fashion}, the  Swiss Roll Cheesecake. See Fig.~\ref{figswissroll2} and Fig.~\ref{figfmnist2} for illustrations of the flow $F_n(x)$ for the   Swiss Roll Cheesecake
and Fashion-MNIST. For these datasets  we observe that
 (1) the  flow $F_n$ unrolls the Swiss Roll Cheesecake (2) the flow $F_n$ expands distances between points that are in different classes and contracts distances between points in the same class (towards archetypes of each class)
(3) at profound depths  ($n=12000$ layers for MNIST and $n=50000$ layers for Fashion-MNIST) the resulting kernel $K_n$ achieves a small average error ($1.5\%$ for MNIST and $10\%$ for Fashion-MNIST) using only $10$ points as interpolation points (i.e. one point for each class)   (4) the incremental data-dependent perturbations  $\epsilon G_{n+1}$ seem to take advantage of a cooperative mechanism appearing as the counterpart of the one associated with adversarial examples \cite{szegedy2013intriguing, owhadiBayesiansirev2013}.

Finally (in Sec.~\ref{secCNN}) we derive an ANN version of the  algorithm by identifying the action of the last layer of the ANN  as that of regressing the data with a kernel parameterized by the weights of all the previous layers (learnt by minimizing $\rho$ or its analogous $L^2$ version). This algorithm is then tested for the MNIST and Fashion-MNIST data sets and shown to contract in-class distances and inter-class distances in a similar manner as above (thereby achieving accuracies comparable to the state of the art with a small number of interpolation points).
For  kernels parameterized by the weights of a given Convolutional Neural Network, minimizing  $\rho$ or its $L^2$ version, appears, for the MNIST and fashion MNIST data sets, to outperform training (the same CNN architecture)  with relative entropy and dropout.

This paper is not aimed at identifying the state of the art algorithm in terms of accuracy nor complexity. It is simply motivated by an attempt to offer some insights (from a numerical approximation perspective) on mechanisms that may be at play in deep learning.

\section{Learning as an interpolation problem}\label{sec2}
It is well understood \cite{poggio2003mathematics} that ``learning techniques
are similar to fitting a multivariate function to a certain number of measurement data'', e.g. solving the following problem.

\begin{Problem}\label{pb1}
Given input/output data $(x_1,y_1),\ldots,(x_N,y_N)\in \X \times \Y$  recover an unknown function $u^\dagger$ mapping $\X$ to $\Y$ such that
\begin{equation}\label{eqkkkhhiuhiu}
u^\dagger(x_i)=y_i\text{ for }i\in \{1,\ldots,N\}\,.
\end{equation}
\end{Problem}

\paragraph{Optimal recovery.}
In the setting of optimal recovery \cite{micchelli1977survey} the ill posed problem \ref{pb1} can be turned into a well posed one by
restricting candidates for $u$ to belong to a space of functions $\B$ endowed with a norm $\|\cdot\|$ and identifying the optimal recovery as the minimizer of the relative error
\begin{equation}\label{eqihdiudehd}
\min_{v} \max_{u}\frac{\|u-v\|^2}{\|u\|^2}\,,
\end{equation}
where the max is taken over $u \in \B$ and the min is taken over candidates in $v\in \B$ such that $v(x_i)=u(x_i)$. Observe that
$\B^*$, the dual space of $\B$, must contain
delta Dirac functions
\begin{equation}
\phi_i(\cdot):=\updelta(\cdot-x_i)\,.
\end{equation}
for the validity of the constraints $u(x_i)=y_i$.
Consider now the case where $\|\cdot\|$ is quadratic, i.e. such that
\begin{equation}
\|u\|^2=[Q^{-1}u, u]\,,
\end{equation}
where $[\phi,u]$ stands for the duality product between $\phi\in\B^*$  and $u\in \B$ and $Q\,:\,\B^*\rightarrow \B$ is a positive symmetric linear bijection (i.e. such that $[\phi,Q\phi]\geq 0$ and $[\phi,Q\varphi]=[\varphi,Q\phi]$ for $\varphi,\phi \in \B^*$).
In that case the optimal solution of \eqref{eqihdiudehd} has the explicit form (writing $y_i$ for $u(x_i)$)
\begin{equation}\label{eqkjehdkjheg}
v^\dagger=\sum_{i,j=1}^N y_i A_{i,j} Q \phi_j\,,
\end{equation}
where $A=\Theta^{-1}$ and $\Theta$ is the $N\times N$ Gram matrix with entries $\Theta_{i,j}=[\phi_i, Q \phi_j]$. Furthermore $v^\dagger$ can also be identified as the minimizer of
\begin{equation}\label{eqvarform}
\begin{cases}
\text{Minimize }\|\psi\|\\
\text{Subject to }\psi\in \B\text{ and }[\phi_i,\psi]=y_i, \quad i\in \{1,\ldots,N\}.
\end{cases}
\end{equation}

\paragraph{Kriging.}
Defining $K$ as the kernel
\begin{equation}
K(x,x')=[\updelta(\cdot-x), Q \updelta(\cdot-x')]\,,
\end{equation}
$(\B,\|\cdot\|)$ can be seen as a Reproducing Kernel Hilbert Space endowed with the norm
\begin{equation}
\|u\|^2=\sup_{\phi \in \B^*} \frac{(\int \phi(x) u(x)\,dx)^2}{\int \phi(x)K(x,y)\phi(y)\,dx\,dy}\,,
\end{equation}
and \eqref{eqkjehdkjheg} corresponds to the classical representer theorem
\begin{equation}\label{eqjehgdgd}
v^\dagger(\cdot)=y^T A  K(x_.,\cdot)\,,
\end{equation}
using the vectorial notations $y^T A  K(x_.,\cdot)=\sum_{i,j=1}^N y_i A_{i,j} K(x_j,\cdot)$ with $A=\Theta^{-1}$ and
 $\Theta_{i,j}=K(x_i,x_j)$.

\paragraph{Gaussian Process Regression numerical approximation games.}
Writing $\xi$ for the centered Gaussian Process with covariance function $K$, \eqref{eqjehgdgd} can also be recovered via Gaussian Process Regression as
\begin{equation}
v^\dagger(x)=\E\big[\xi(x)\mid \xi(x_i)=y_i\big]\,.
\end{equation}
This link between Numerical Approximation and Gaussian Process Regression emerges naturally by viewing \eqref{eqihdiudehd} as an adversarial zero sum game \cite{OwhScobook2018,  OwhadiMultigrid:2015, gamblet17, SchaeferSullivanOwhadi17} between two players (I and II where I tries to maximize the relative error and II tries to minimize it after seeing the values of $u$ at the points $x_i$) and observing that $\xi$ and $v$ are optimal mixed/randomized strategies for players I and II (forming a saddle point for the minimax lifted to measures over functions).

\section{What is a good kernel?}\label{sec3}
Although the optimal recovery of $u^\dagger$ has a well established theory, it relies on the prior specification of a quadratic norm $\|\cdot\|$ or equivalently of a kernel $K$. In practical applications the performance of the interpolant \eqref{eqjehgdgd} (e.g. when employed in a classification problem) is sensitive to the choice of $K$. How should $K$ be selected to achieve generalization?
Although ANNs \cite{lecun2015deep} seem to address this question  (by performing variants of the interpolation \eqref{eqjehgdgd} with the last layer of the network using
   a kernel $K$ parameterized by the weights of the previous layers and learnt by adjusting those weights) they remain difficult to analyze and the introduction of regularization steps (such as dropout or early stop) introduced to achieve generalization appear to be discovered through a laborious process of trial and error \cite{zhang2016understanding}.

   Is there a systematic way of identifying a good kernel? What is good kernel?

 We will now explore these questions from  the  perspective  of interplays between numerical approximation and inference \cite{OwhScobook2018} and the simple premise that a kernel must be good if the number of points $N$ used to perform the interpolation of data can be reduced to $m=\operatorname{round}(N/2)$ without significant loss in accuracy (measured using the intrinsic RKHS norm associated with the kernel).

 To label the $m$ sub-sampled (test) data points,
  let $\s(1),\ldots,\s(m)$ be a selection of $m$ distinct elements of $\{1\ldots,N\}$. Observe that $\{x_{\s(1)},\ldots,x_{\s(m)}\}$ forms a strict subset of  $\{x_1,\ldots,x_N\}$.
  Write  $v^\s$ for the optimal recovery of $u^\dagger$ upon seeing its values at the points $x_{\s(1)},\ldots,x_{\s(m)}$, and observe that  $v^\s(\cdot)=\sum_{i=1}^m \bar{y}_i \bar{A}_{i,j} K(x_{\s(j)},i)$
 with $\bar{y}_i=y_{\s(i)}$ and $\bar{A}=\bar{\Theta}^{-1}$ with $\bar{\Theta}_{i,j}=\Theta_{\s(i),\s(j)}$.
Let $\pi$ be the corresponding $m\times N$ sub-sampling matrix defined by
$ \pi_{i,j}=\delta_{\s(i), j}$
and observe that
\begin{equation}\label{vsigm}
v^\s= y^T \tilde{A} K(x_.,\cdot)
\end{equation}
with $\tilde{A}=\pi^T \bar{A}\pi$ and $\bar{A}=(\pi \Theta \pi^T)^{-1}$.
\begin{Proposition}\label{propjhgfytf}
For $v^\dagger$ and $v^\s$ defined as in \eqref{eqjehgdgd} and \eqref{vsigm}, we have
\begin{equation}\label{eqkjhedjhdg}
\|v^\dagger-v^\s\|^2=y^T A y-y^T \tilde{A} y\,.
\end{equation}
\end{Proposition}
\begin{proof}
Proposition \ref{propjhgfytf} is particular case of \cite[Prop.~13.29]{OwhScobook2018}.
The proof follows simply from $\|v^\dagger\|^2=y^T A y$, $\|v^\s\|^2=\bar{y}^T \bar{A} \bar{y}=y^T \tilde{A} y$ and the orthogonal decomposition
$\|v^\dagger\|^2=\|v^\s\|^2+\|v^\dagger-v^\s\|^2$ implied by the fact that
$v^\s$ is the minimizer of $\|\psi\|^2$ subject to the constraints $[\phi_{\s(i)},\psi]=y_{\s(i)}$ and that $v^\dagger$ satisfies those constraints.
\end{proof}

Let $\rho$ be the ratio
\begin{equation}\label{eqkedhkdh}
\rho:=\frac{\|v^\dagger-v^\s\|^2}{\|v^\dagger\|^2}\,.
\end{equation}
Note that  a value of $\rho$ close to zero indicates
that $v^\s$ is a good approximation of $v^\dagger$ (and that most of the energy of $v^\dagger$ is contained in $v^\s$) which is a desirable condition for the kernel $K$ to achieve generalization. Furthermore, Prop.~\ref{propjhgfytf} implies that
$\rho\in [0,1]$ and
\begin{equation}
\rho=1-\frac{y^T \tilde{A} y}{y^T A y}\,.
\end{equation}
Fixing $y$ and $\pi$, $\rho$ can be seen as a function of $A$ which we will write $\rho(A)$. Since $A=\Theta^{-1}$, $\rho$ can also be viewed as a function of $\Theta$ which, abusing notations, we will write $\rho(\Theta)$.
Motivating by the application of $\rho$ to the ordering of space of kernels (a small $\rho$ being indicative of a good kernel) we will, in the following proposition, compute its Fr\'echet derivative with respect to small perturbations of $A$ or of $\Theta$.

\begin{Proposition}\label{propekjdhgud}
Write $z:=A^{-1}\tilde{A} y$ with $\tilde{A}:=\pi^T (\pi A^{-1} \pi^T)^{-1}\pi$ defined as above.\footnote{The operator $P:=\tilde{A}A^{-1}$ is a projection with $\Img(P)=\Img(\pi^T)$ and $\Ker(P)=A \Ker(\pi)$ and
from the perspective of numerical homogenization  $\tilde{A}$ can be interpreted as the homogenized version of $A$ \cite[Sec.~13.10.3]{OwhScobook2018}, see \cite[Chap.~13.10]{OwhScobook2018} for further  geometric properties.}
It holds true that
\begin{equation}\label{eqkejdkddjdj}
\rho(A+\epsilon S)=\rho(A)+\epsilon \frac{(1-\rho(A)) y^T S y  - z^T S z }{y^T A y} +\mathcal{O}(\epsilon^2)\,,
\end{equation}
and, writing\footnote{Note that $\hat{z}=\Theta^{-1}z=\tilde{A} y$.} $\hat{y}:=\Theta^{-1} y$ and $\hat{z}:=\pi^T (\pi \Theta \pi^T)^{-1}\pi y$,
\begin{equation}\label{eqlkejdhkdhjkdh}
\rho(\Theta+\epsilon T)=\rho(\Theta)-\epsilon \frac{(1-\rho(\Theta)) \hat{y}^T T \hat{y}  - \hat{z}^T T \hat{z} }{\hat{y}^T \Theta \hat{y}}+\mathcal{O}(\epsilon^2)\,.
\end{equation}
\end{Proposition}
\begin{proof}
Observe that
\begin{equation}
\rho(A+\epsilon S)=1-\frac{y^T \pi^T (\pi (A+\epsilon S)^{-1}\pi^T)^{-1}\pi y}{y^T (A+\epsilon S) y}
\end{equation}
and recall the approximation
\begin{equation}\label{eqhgjhgjh}
(A+\epsilon S)^{-1}=A^{-1}-\epsilon A^{-1}S A^{-1}+\mathcal{O}(\epsilon^2)\,.
\end{equation}
 \eqref{eqkejdkddjdj} then follows from straightforward calculus. The proof of \eqref{eqlkejdhkdhjkdh} is identical and can also be obtained from \eqref{eqkejdkddjdj} and the first order approximation $(\Theta+\epsilon T)^{-1}=\Theta^{-1}-\epsilon \Theta^{-1}T \Theta^{-1}+\mathcal{O}(\epsilon^2)$.
\end{proof}

\section{The algorithm with a parametric family of kernels}\label{secfamker}
Let $\cW$ be a finite dimensional linear space and let
 $K(x,x',W)$ be a family of kernels parameterized by $W\in \cW$. Let $N_f\leq N$ and $N_c=\operatorname{round}(N_f/2)$.
Let $\s_f(1),\ldots,\s_f(N_f)$ be a selection of $N_f$ distinct elements of $\{1,\ldots,N\}$. Let $\s_{c}(1),\ldots,\s_c(N_c)$ be a selection of $N_c$ distinct elements of $\{1,\ldots,N_f\}$.
Let $\pi$ be the  corresponding $N_c\times N_f$ sub-sampling matrix defined by
$ \pi_{i,j}=\delta_{\s_c(i), j}$. Let $y_f\in \R^{N_f}$ and $y_c\in \R^{N_c}$ be the corresponding subvectors of $y$ defined by
$y_{f,i}=y_{\s_f(i)}$ and $y_{c,i}=y_{f,\s_c(i)}$.

 Using the notations of sections  \ref{sec2} and
\ref{sec3} write $\Theta(W)$ for the $N_f\times N_f$ matrix with entries $\Theta_{i,j}=K(x_{s_f(i)},x_{s_f(j)},W)$ and let
\begin{equation}\label{eqjhgjgyuyguy0}
\rho(W,\s_f,\s_c):=1-\frac{ y_c^T(\pi\Theta \pi^T)^{-1} y_c}{y_f^T \Theta^{-1} y_f}\,.
\end{equation}

The following corollary derived from  Prop.~\ref{propekjdhgud} allows us to compute the gradient of $\rho$ respect to $W$.
\begin{Corollary}\label{corpropekjdhgud}
Write $\Theta:=\Theta(W)$, $\hat{y}:=\Theta^{-1} y_f$ and $\hat{z}:=\pi^T (\pi \Theta \pi^T)^{-1}\pi y_f$.  Write $W_i$ for the entries of the vector $W$. It holds true that
\begin{equation}\label{eqlkejdhkdhjkdh2}
\partial_{W_i} \rho(W)=-\frac{(1-\rho(W)) \hat{y}^T (\partial_{W_i} \Theta(W)) \hat{y}  - \hat{z}^T (\partial_{W_i} \Theta(W)) \hat{z} }{y_f^T \Theta^{-1} y_f}\,.
\end{equation}
\end{Corollary}
\begin{proof}
Prop.~\ref{propekjdhgud} implies that
\begin{equation}\label{eqlkejdhkdhjkdh3}
\rho(W+\epsilon W')=\rho(W)-\epsilon \frac{(1-\rho(W)) \hat{y}^T T \hat{y}  - \hat{z}^T T \hat{z} }{y_f^T \Theta^{-1} y_f}+\mathcal{O}(\epsilon^2)\,,
\end{equation}
with  $T=(W')^T \nabla_W \Theta(W)$, which proves the result.
\end{proof}

The purpose of Algorithm \ref{alglearnkernel} is to learn the parameters $W$ (of the kernel $K$) from the data.
The value of $N_f$ (and hence $N_c$) corresponds to the size of a batch.
The initialization of
$W$ in step \ref{step1} may be problem dependent or at random.

\begin{algorithm}[h]
\caption{Learning $W$ in the $K(\cdot,\cdot,W)$.}\label{alglearnkernel}
\begin{algorithmic}[1]
\STATE\label{step1} Initialize $W$
\REPEAT
\STATE Select $\s_f(1),\ldots,\s_f(N_f)$ out of $\{1,\ldots,N\}$.
\STATE Select $\s_c(1),\ldots,\s_c(N_c)$ out of $\{1,\ldots,N_f\}$.
\STATE $W'=-\nabla_W \rho(W,\s_f,\s_c)$
\STATE  $W=W +\epsilon W'$
\UNTIL End criterion
\end{algorithmic}
\end{algorithm}

\section{A simple PDE model}\label{secjhdkhj33}
To motivate, illustrate and study the proposed approach, it is useful to start with an application to the following simple PDE model amenable to detailed analysis \cite{OwhScobook2018}.
Let $u$ be the solution of
\begin{equation}\label{eqnscalarhgyprotgtoa}
\begin{cases}
    -\diiv \big(a(x)  \nabla u(x)\big)=f(x) \quad  x \in \Omega;  \\
    u=0 \quad \text{on}\quad \partial \Omega\,,
    \end{cases}
\end{equation}
where $\Omega\subset \R^d$,  is a regular subset  and $a$ is a uniformly elliptic  symmetric matrix with entries in $L^\infty(\Omega)$.
Write $\L:=-\diiv(a\nabla\cdot)$ for the corresponding linear bijection from $\H^1_0(\Omega)$ to $\H^{-1}(\Omega)$.

In this proposed simple application we seek to recover the solution of  \eqref{eqnscalarhgyprotgtoa} from the data $(x_i,y_i)_{1\leq i \leq N}$ and the information $u(x_i)=y_i$.
If the conductivity $a$ is known  then \cite{OwhScobook2018, gamblet17} interpolating the data with the kernel
(1) $\L^{-1}$  leads to a recovery that is minimax optimal in the (energy) norm $\|u\|^2=\int_{\Omega}(\nabla u)a \nabla u$ ($d=1$ is required to ensure the continuity of the kernel).
(2) $(\L^T \L )^{-1}$ leads to a recovery that is minimax optimal in the  norm $\|u\|=\|\diiv(a\nabla u)\|_{L^2(\Omega)}$ ($d\leq 3$, the recovery is equivalent to interpolating with Rough Polyharmonic Splines \cite{OwhadiZhangBerlyand:2014}).
 (3) $(\L^T \Delta \L )^{-1}$ leads to a recovery that is minimax optimal in the  norm $\|u\|=\|\diiv(a\nabla u)\|_{\H^1_0(\Omega)}$ ($d\leq 5$).
  \begin{figure}[h]
\begin{center}
\includegraphics[width= \textwidth]{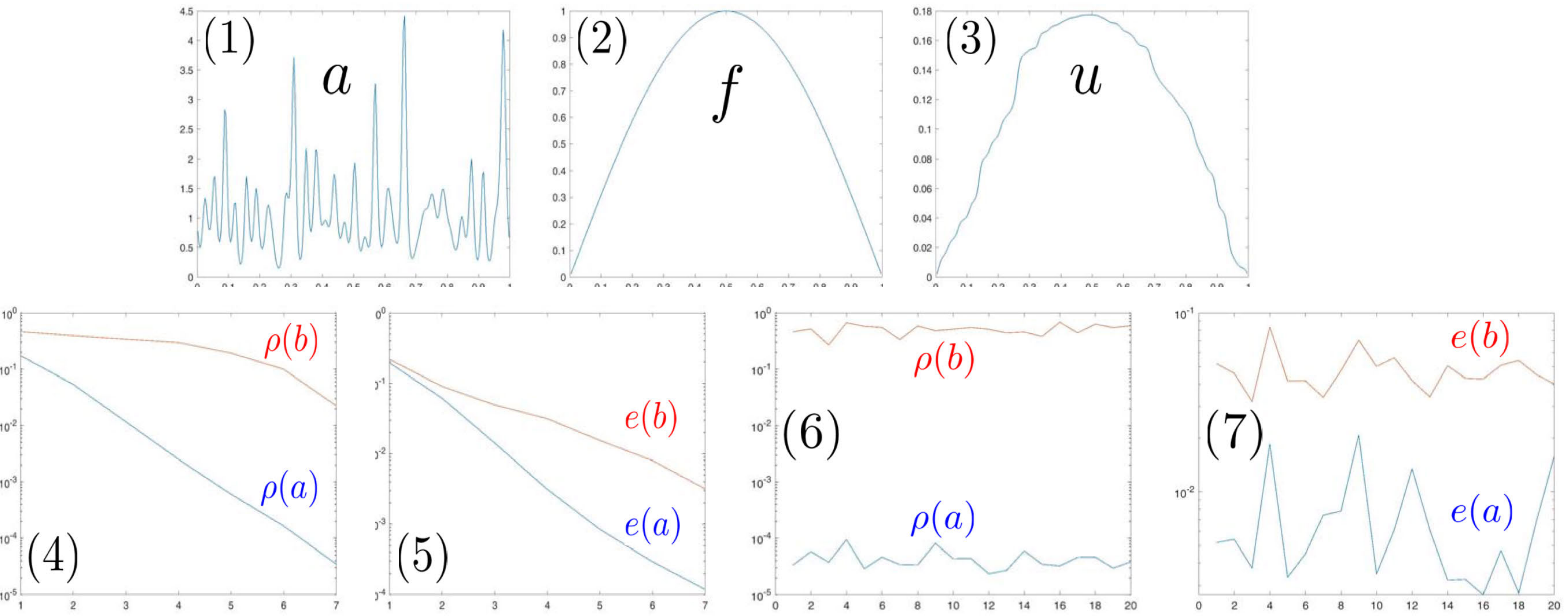}
\caption{(1) $a$ (2) $f$ (3) $u$ (4) $\rho(a)$ and $\rho(b)$ vs $k$, geometric (5)  $e(a)$ and $e(b)$ vs $k$, geometric
(6) $\rho(a)$ and $\rho(b)$ vs $k$, random  (5)  $e(a)$ and $e(b)$ vs random realization.}
\label{figa}
\end{center}
\end{figure}

 Which kernel should be used for the recovery of $u$ when the conductivity $a$ is unknown? Consider the case when
 $d=1$ and $\Omega$ is the interval $(0,1)$. For $b\in \L^\infty(\Omega)$ with $\operatorname{essinf}_{\Omega}(b)>0$
  write $G_b$ for the Green's function of the operator $-\diiv(b\nabla \cdot)$ mapping
 $\H^1_0(\Omega)$ to $\H^{-1}(\Omega)$. Observe that the $\{G_b|b\}$ is a set of kernels parameterized $b$ and any kernel in that set could be used to interpolate the data. Which one should we pick? The answer proposed in Sec. \ref{sec3} and \ref{secfamker} is to
 use the ordering induced by $\rho$ to select the kernel.

 \paragraph{Fig.~\ref{figa}}
 provides a numerical illustration of that ordering. In that example $\Omega$ is discretized over $2^8$ equally spaced interior points (and piecewise linear tent finite elements) and Fig.~\ref{figa}.1-3 shows $a$, $f$ and $u$.  For $k\in \{1,\ldots,8\}$ and $i\in \I^{(k)}:=\{1,\ldots, 2^{k}-1\}$ let $x_i^{(k)}=i/2^k$ and write $v^{(k)}_b$ for the interpolation of the data $(x_i^{(k)},u(x_i^{(k)}))_{i\in \I^{(k)}}$ using the kernel $G_b$ (note that $v^{(8)}_b=u$).
Let $\|v\|_b$ be the energy norm $\|v\|^2_b=\int_{\Omega}(\nabla v)^Tb \nabla v$.
Take $b\equiv 1$. Fig.~\ref{figa}.4 shows (in semilog scale) the values of
$\rho(a)=\frac{\|v^{(k)}_a-v^{(8)}_a\|_a^2}{\|v^{(8)}_a\|_a^2}$ and $\rho(b)=\frac{\|v^{(k)}_b-v^{(8)}_b\|_b^2}{\|v^{(8)}_b\|_b^2}$ vs $k$.
Note that the value of ratio $\rho$ is much smaller when the kernel $G_a$ is used for the interpolation of the data. The geometric decay $\rho(a)\leq C 2^{-2k} \frac{\|f\|_{L^2(\Omega)}}{\|u\|_a^2}$ is well known and has been extensively studied in  Numerical Homogenization \cite{OwhScobook2018}.

Fig.~\ref{figa}.5 shows (in semilog scale) the values of the average prediction errors $e(a)$ and $e(b)$ (vs $k$) defined (after normalization) to be proportional to
$\|v^{(k)}_a(x)-u(x)\|_{L^2(\Omega)}$ and $\|v^{(k)}_b(x)-u(x)\|_{L^2(\Omega)}$.
Note again that the prediction error is much smaller when the kernel $G_a$ is used for the interpolation.

Now let us consider the case where the interpolation points form a random subset of the discretization points.
Take $N_f=2^7$ and $N_c=2^6$. Let
 $X=\{x_1,\ldots,x_{N_f}\}$ be a subset $N_f$ distinct points of (the discretization points) $\{i/2^{8}|i\in \I^{(8)}\}$ sampled with uniform distribution. Let $Z=\{z_1,\ldots,z_{N_c}\}$ be a subset of $N_c$ distinct points of $X$ sampled with uniform distribution.
Write $v^f_b$ for the interpolation of the data $(x_i,u(x_i))$ using the kernel $G_b$ and write
 $v^c_b$ for the interpolation of the data $(z_i,u(z_i))$ using the kernel $G_b$.
Fig.~\ref{figa}.6 shows in (semilog scale) $20$ independent random realizations of the values of
$\rho(a)=\|v^f_a-v^c_a\|_a^2/\|v^f_a\|_a^2$ and $\rho(b)=\|v^f_b-v^c_b\|_b^2/\|v^f_b\|_b^2$.
Fig.~\ref{figa}.7 shows in (semilog scale) $20$ independent random realizations of the values of the prediction errors
$e(a)\propto\|u-v^c_a\|_{L^2(\Omega)}$ and $e(b)\propto\|u-v^c_b\|_{L^2(\Omega)}$. Note again that the values of $\rho(a), e(a)$ are consistently and significantly lower than those of $\rho(b), e(b)$.

\begin{figure}[h]
\begin{center}
\includegraphics[width= \textwidth]{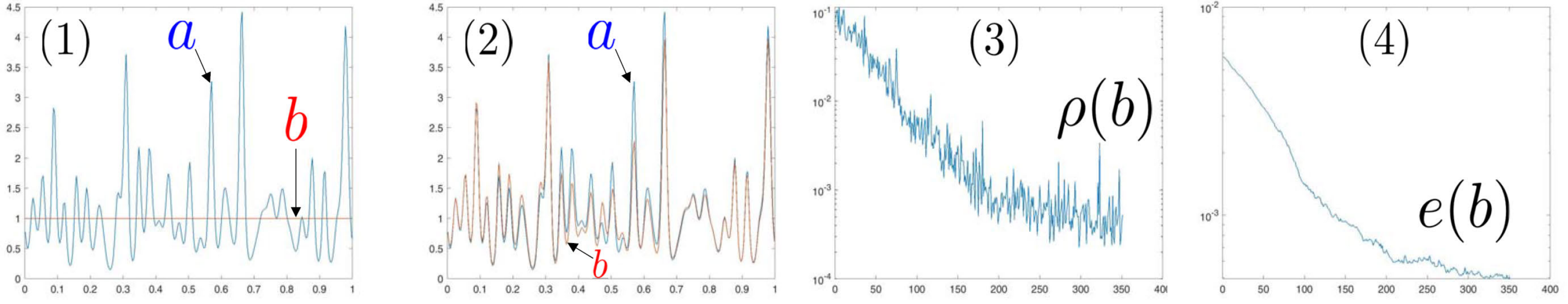}
\caption{(1) $a$ and $b$ for $n=1$ (2) $a$ and $b$ for $n=350$ (2)  $\rho(b)$ vs $n$ (4) $e(b)$ vs $n$.}
\label{figaguess}
\end{center}
\end{figure}

 \paragraph{Fig.~\ref{figaguess}} provides a numerical illustration of an implementation of Alg.~\ref{alglearnkernel} with
 $N_f=N=2^7$, $N_c=2^6$ and $n_c=1$.
  In this implementation $a, f$ and $u$ are as in Fig.~\ref{figa}.1-3. The training data corresponds to $N_f$ points
  $X=\{x_1,\ldots,x_{N_f}\}$ uniformly sampled (without replacement) from $\{i/2^{8}|i\in \I^{(8)}\}$ (Since $N=N_f$ these points remain fixed during the execution the of the algorithm).
  $n$. The purpose of the algorithm is to learn the kernel $G_a$ in the set of kernels $\{G_{b(W)}|W\}$ parameterized by the vector $W$ via
  \begin{equation}
  \log b(W)=\sum_{i=1}^{2^6} (W^c_i \cos(2\pi i x )+ W^s_i \sin(2\pi i x))\,.
  \end{equation}
  Using $n$ to label its progression, Alg.~\ref{alglearnkernel} is initialized at $n=1$ with the guess $b\equiv 1$ (i.e. $W\equiv 0$) (Fig.~\ref{figaguess}.1). At each step ($n\rightarrow n+1$) the algorithm performs the following operations:
  \begin{enumerate}
  \item Select $N_c$ points
  $Z=\{z_1,\ldots,z_{N_c}\}$ uniformly sampled (without replacement) from $X$.
  \item Write $v^f_b$ and $v^c_f$ for the interpolation of the data $(x_i,u(x_i))$ and $(z_i,u(z_i))$ using the kernel $G_b$,
  and $\rho(W)=\|v^f_b-v^c_b\|_b^2/\|v^f_b\|_b^2$. Compute the gradient $\nabla_W \rho(W)$ using Cor.~\ref{corpropekjdhgud} (and
  the identity $\partial_{W_i}\Theta(W)=   -\pi_0 (A_0(W))^{-1} \partial_{W_i} A_0(W)   (A_0(W))^{-1} \pi_0^T $ for
  $\Theta(W)=   \pi_0 (A_0(W))^{-1} \pi_0^T $).
  \item Update $W\rightarrow W-\lambda \nabla_W \rho(W)$ (with $\lambda\propto 0.01/\|\nabla_W \rho(W)\|_{L^2}$).
  \end{enumerate}
  Fig.~\ref{figaguess}.2 shows the value of $b$ for $n=350$. Fig.~\ref{figaguess}.3 shows the value of $\rho(b)$  vs $n$. Fig.~\ref{figaguess}.4 shows the value of the  prediction error $e(b)\propto\|u-v^c_b\|_{L^2(\Omega)}$ vs $n$. The lack of smoothness of the plots of $\rho(b), e(b)$ vs $n$ originate from the re-sampling of the set $Z$ at each step $n$.

\begin{comment}

\begin{figure}[h]
\begin{center}
\includegraphics[width= \textwidth]{./fig/u7full}
\caption{Multiresolution decomposition of the solution $u$ of \eqref{eqnscalarhgyprotgtoa}}
\label{figu7full}
\end{center}
\end{figure}
\end{comment}

  \begin{figure}[h]
\begin{center}
\includegraphics[width= \textwidth]{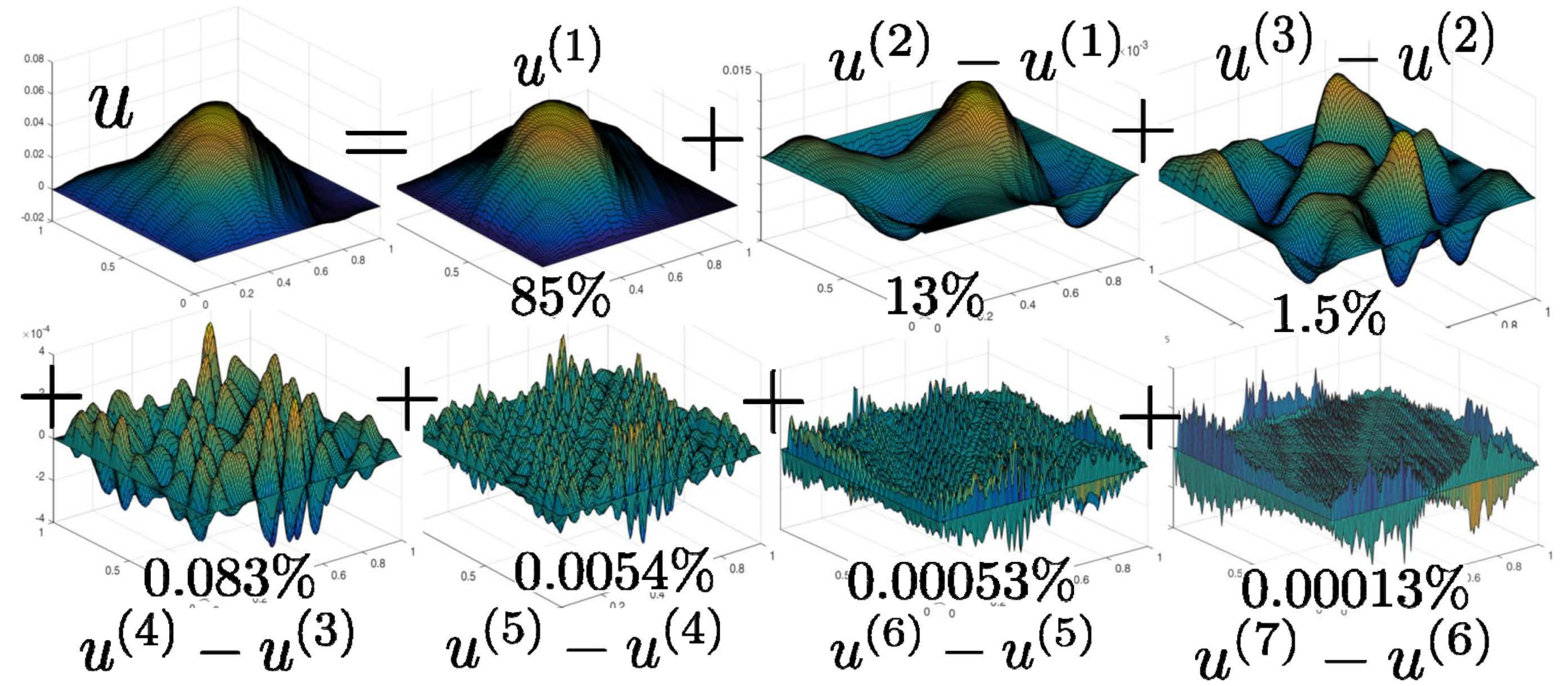}
\caption{$u$ and $u^{(k)}-u^{(k-1)}$. Number below sub-figures show relative energy content $\|u^{(k)}-u^{(k-1)}\|^2/\|u\|^2$.   Used from forthcoming book \cite{OwhScobook2018} with permission from Cambridge University Press.}
\label{figu7}
\end{center}
\end{figure}

\begin{Remark}
For $d\geq 1$, let $\Omega=(0,1)^d$ and for $k\geq 1$, let
$\tau_i^{(k)}$ be a nested (in $k$) hierarchy of sub-cubes of $(0,1)^d$  with locations indexed by $i$. For $\L=-\diiv(a\nabla)$ let $\xi\sim \cN(0,\L^{-1})$ and let $u^{(k)}:=\E[\xi\mid \int_{\tau_i^{(x)}}\xi = \int_{\tau_i^{(x)}}\xi u \text{ for all }i]$.
Fig.~\ref{figu7} shows (for $d=2$) the corresponding increments $u^{(k)}-u^{(k-1)}$ and the relative energy content $\|u^{(k)}-u^{(k-1)}\|^2/\|u\|^2$ of each increment for a solution of \eqref{eqnscalarhgyprotgtoa} with $f\in L^2(\Omega)$.
The quick decay of $\|u^{(k)}-u^{(k-1)}\|^2/\|u\|^2 $  with respect to $k$ illustrates the accuracy of the Green's function of
\eqref{eqnscalarhgyprotgtoa} used as a kernel for  interpolating partial linear measurements made on solutions of \eqref{eqnscalarhgyprotgtoa}. This numerical homogenization phenomenon \cite{OwhScobook2018, OwhadiMultigrid:2015} is one motivation for minimizing $\rho$ in the kernel identification problem described above ($\|u^{(k)}-u^{(k-1)}\|^2/\|u\|^2 \approx \rho_k$ with $\rho_k:=\|u^{(k)}-u^{(k-1)}\|^2/\|u^{(k)}\|^2$).
\end{Remark}

\section{Kernel Flows (KF)}\label{secKernelFlow}

%Kernel interpolation network

%Numerical Approximation Network (NANet)
\subsection{Non parametric family of kernels and bottomless networks without guesswork}
Composing a symmetric positive kernel with a function produces a symmetric positive kernel \cite{cristianini2000introduction}.
We will now use this property to learn a kernel from the data within a non-parametric family of kernels constructed by composing layers of functions.
For $n\geq 2$ let $G_n\,:\,\X\rightarrow \X$ ($\X$ is the input space mentioned Pb.~\ref{pb1}) be a sequence of functions determining the layers of this network. Let $\epsilon>0$ be a small parameter, let $F_1:=I_d$ be the identity function and for $n\geq 2$, let $F_n$ be the sequence of functions inductively defined by
\begin{equation}\label{eqklejdhkdj}
F_{n+1}=(I_d+\epsilon G_{n+1})\circ F_n\,.
\end{equation}
Let $K_n$ be the sequence of symmetric positive kernels obtained by composing  a kernel $K_1$ with this sequence of functions, i.e. $K_n(x,x')=K_1(F_n(x),F_n(x'))$ and
\begin{equation}\label{eqlkedjdldh}
K_{n+1}(x,x')=K_n\big(x+\epsilon G_{n+1}(x),x'+\epsilon G_{n+1}(x')\big)\,.
\end{equation}
Our purpose is to use the training data
$(x_1,y_1),\ldots,(x_N,y_N)\in \X \times \Y$ to learn the functions $G_1,\ldots,G_{n^*}$ and then
approximate $u^\dagger$ with $u_{n^*}$ obtained by interpolating a subset of the training data with $K_{n^*}$.

When applied to a classification problem with $n=1$ the proposed algorithm is a support-vector network  (in the sense of \cite{cortes1995support}) with kernel $K_1$. As $n$ progresses the algorithm incrementally modifies the kernel via small perturbations of the identity operator ($\epsilon G_{n+1}$ is reminiscent of the residual term of deep residual networks \cite{he2016deep}).
Since the training does not require any back propagation, achieving profound depths (with $10000$ layers or more) is not difficult (since training is akin to simulating a stochastic dynamical system the network is essentially bottomless) and one purpose of this section is to explore  properties of such bottomless networks (see \cite{huang2017densely} for a review of the motivations/challenges associated with the exploration of very deep networks).

\subsection{The algorithm}\label{refdklejdhhdkj}
We will adapt Algorithm \ref{alglearnkernel} to learn the functions $G_1,\ldots,G_{n},\ldots$ by induction over $n$.
As in Sec.~\ref{secfamker} let  $N_f\leq N$ and $N_c=\operatorname{round}(N_f/2)$.
\paragraph{ For $n=1$} let $x^{(n)}_i:=x_i$ for $i\in \{1,\ldots,N\}$.

\paragraph{Let $n\geq 1$.} Assume $x^{(n)}_1,\ldots,x^{(n)}_N$ to be known.
Let $\s_{f,n}(1),\ldots,\s_{f,n}(N_f)$ be  $N_f$ distinct elements of $\{1,\ldots,N\}$ obtained through random sampling (with uniform distribution) without replacement. Let $\s_{c,n}(1),\ldots,\s_{c,n}(N_c)$ be  $N_c$ distinct elements of $\{1,\ldots,N_f\}$ also obtained through random sampling (with uniform distribution) without replacement. Let $\pi$ be the  corresponding $N_c\times N_f$ sub-sampling matrix defined by
$ \pi_{i,j}^{(n)}=\delta_{\s_{c,n}(i), j}$. Let $y_f^{(n)}\in \R^{N_f}$ and $y_c^{(n)}\in \R^{N_c}$ be the corresponding subvectors of $y$ defined by
$y_{f,i}^{(n)}=y_{\s_{f,n}(i)}$ and $y_{c,i}^{(n)}=y_{f,\s_{c,n}(i)}$.
For $i\in \{1,\ldots,N_f\}$ let
 $x^{(n)}_{f,i}:=x_{\s_{f,n}(i)}^{(n)}$ and
write $\Theta^{(n)}$ for the $N_f\times N_f$ matrix with entries
 \begin{equation}
\Theta^{(n)}_{i,j}=K_1(x^{(n)}_{f,i},x^{(n)}_{f,j})\,,
\end{equation}
and let
\begin{equation}\label{eqjhgjgyuyguy}
\rho(n):=1-\frac{ (y_c^{(n)})^T(\pi^{(n)}\Theta^{(n)} (\pi^{(n)})^T)^{-1} y_c^{(n)}}{(y_f^{(n)})^T (\Theta^{(n)})^{-1} y_f^{(n)}}\,.
\end{equation}
Let $\hat{y}^{(n)}_f:=(\Theta^{(n)})^{-1} y_f^{(n)}$,  $\hat{z}^{(n)}_f:= (\pi^{(n)})^T (\pi^{(n)}\Theta^{(n)}(\pi^{(n)})^T)^{-1}\pi^{(n)} y_f^{(n)}$
and for $i\in \{1,\ldots,N_f\}$ let
\begin{equation}\label{eqkjjdjkhedhd}
\begin{split}
\hat{g}^{(n)}_{f,i}:=2\frac{(1-\rho(n)) \hat{y}^{(n)}_{f,i}  (\nabla_x K_1)(x_{f,i}^{(n)},x_{f,\cdot}^{(n)}) \hat{y}^{(n)}_f
 - \hat{z}^{(n)}_{f,i}  (\nabla_x K_1)(x_{f,i}^{(n)},x_{f,\cdot}^{(n)}) \hat{z}^{(n)}_f}{y_f^T (\Theta^{(n)})^{-1} y_f}
\end{split}
\end{equation}
Let $G_{n+1}$ be the function obtained by interpolating the data $(x^{(n)}_{f,i}, \hat{g}^{(n)}_{f,i})$ with the kernel $K_1$, i.e.
\begin{equation}\label{eqjhgyuygygtft}
G_{n+1}(x)= (\hat{g}_{f,\cdot}^{(n)})^T \big(K_1(x_{f,\cdot}^{(n)},x_{f,\cdot}^{(n)})\big)^{-1} K_1(x_{f,\cdot}^{(n)},x)\,.
\end{equation}
Note that $G^{(n+1)}(x_{f,i}^{(n)})=\hat{g}^{(n)}_{f,i}$.
For $i\in \{1,\ldots,N\}$, let
\begin{equation}
x^{(n+1)}_i=x^{(n)}_i+\epsilon G_{n+1}(x^{(n)}_i)\,.
\end{equation}

\begin{Remark}\label{rmkksjhddg7}
In the description above the input space $\X$ (of the function $u$ to be interpolated) is assumed to a finite-dimensional vector space and the output space $\Y$ (of the function $u$ to be interpolated) is assumed to be contained in the real line $\R$.
If the output space $\Y$ is a finite-dimensional vector space (e.g. $\R^{d_\Y}$) then
$y_{f}^{(n)}$ and $y_{c}^{(n)}$ are $N_f\times d_{\Y}$ and $N_c\times d_{\Y}$ matrices and
 \eqref{eqjhgjgyuyguy} and \eqref{eqkjjdjkhedhd} must be replaced by
 \begin{equation}\label{eqjhgjgyuyguymod}
\rho(n):=1-\frac{ \Tr\big[(y_c^{(n)})^T(\pi^{(n)}\Theta^{(n)} (\pi^{(n)})^T)^{-1} y_c^{(n)}\big]}{\Tr\big[(y_f^{(n)})^T (\Theta^{(n)})^{-1} y_f^{(n)}\big]}\,.
\end{equation}
 and
\begin{equation}\label{eqkjjdjkhedhdmod}
\begin{split}
\hat{g}^{(n)}_{f,i}:=2\frac{\Tr\big[(1-\rho(n)) \hat{y}^{(n)}_{f,i}  (\nabla_x K_1)(x_{f,i}^{(n)},x_{f,\cdot}^{(n)}) \hat{y}^{(n)}_f
 - \hat{z}^{(n)}_{f,i}  (\nabla_x K_1)(x_{f,i}^{(n)},x_{f,\cdot}^{(n)}) \hat{z}^{(n)}_f\big]}{\Tr\big[y_f^T (\Theta^{(n)})^{-1} y_f\big]}\,,
\end{split}
\end{equation}
where, in \eqref{eqkjjdjkhedhdmod}, $\hat{y}^{(n)}_{f,i}\in \R^{d_{\Y}}$, $\hat{y}^{(n)}_f\in \R^{N_f\times d_{\Y}}$,
$(\nabla_x K_1)(x_{f,i}^{(n)},x_{f,\cdot}^{(n)})\in \R^{d_{\X}\times N_f}$ (writing $d_\X$ for the dimension of the input space, $\nabla_x K_1$ refers to the gradient over the first component of $K_1$). The product results in a $d_\Y \times d_\X  \times d_{\Y}$ tensor and the trace  (taken with respect to the $d_Y$ dimensions) results in a vector in $\R^{d_\X}$, i.e. writing $x_s$ for the $s$th entry of $x$,
$
(\hat{g}^{(n)}_{f,i})_s=2\frac{\sum_{l=1}^{d_\Y}\sum_{t=1}^{N_f} (1-\rho(n)) (\hat{y}^{(n)}_{f})_{i,l}  \partial_{x_s} K_1(x_{f,i}^{(n)},x_{f,t}^{(n)}) (\hat{y}^{(n)}_{f})_{t,l}
 - (\hat{z}^{(n)}_{f})_{i,l}  \partial_{x_s} K_1(x_{f,i}^{(n)},x_{f,t}^{(n)}) (\hat{z}^{(n)}_f)_{t,l}}{\Tr\big[y_f^T (\Theta^{(n)})^{-1} y_f\big]}
$.

This simple modification (via the trace operator) is equivalent to endowing the space of functions $v=(v_1,\ldots,v_{d_\Y})$ mapping $\X$ to $\Y$ with the RKHS norm
$\|v\|^2=\sum_{i=1}^{d_\Y} \|v_i\|^2 $ with $\|v_i\|^2= \sup_{\phi} \frac{(\int_{\X} \phi(x) v_i(x)\,dx)^2}{\int_{\X^2} \phi(x)K_n(x,x') \phi(x')\,dx\,dx' } $ and using that norm to compute $\rho$ and its  gradient.
\end{Remark}

\subsection{Rationale of the algorithm}\label{secrat}
Observe that the algorithm is randomized through the random samples $\s_{f,n}(i)$ and $\s_{c,n}(i)$ (taking values in the training data).
Observe also that, given those random samples, the functions $G_n, F_n$ and kernels $K_n$ are entirely determined by the values of the learning parameters $\hat{g}^{(n)}_{f,i}=G^{(n+1)}(x_{f,i}^{(n)})$.

As in Alg.~\ref{alglearnkernel},  the $\hat{g}^{(n)}_{f,i}$  are selected in \eqref{eqkjjdjkhedhd} in the direction of the  gradient descent of the ratio $\rho$: we  apply Prop.~\ref{propekjdhgud} to compute the Frech\'{e}t derivative of $\rho(n)$ \eqref{eqjhgjgyuyguy} with respect to small perturbations $x^{(n)}_{f,i}+\epsilon g^{(n)}_{f,i}$ to the  $x^{(n)}_{f,i}$ and select
the $g^{(n)}_{f,i}$ in the  direction of the  gradient descent.

More precisely let $g^{(n)}_{f,i}$ be candidates for the values of
$G^{(n+1)}(x_{f,i}^{(n)})$ and write
$\tilde{x}^{(n+1)}_{f,i}=x^{(n)}_{f,i}+\epsilon g^{(n)}_{f,i}$. Write $\tilde{\Theta}^{(n+1)}$ for the $N_f\times N_f$ matrix with entries
 \begin{equation}
\tilde{\Theta}^{(n+1)}_{i,j}=K_1(\tilde{x}^{(n+1)}_{f,i},\tilde{x}^{(n+1)}_{f,j})\,,
\end{equation}
and let
\begin{equation}\label{eqjhgjgyuyguy2}
\tilde{\rho}(n+1):=1-\frac{ (y_c^{(n)})^T(\pi^{(n)} \tilde{\Theta}^{(n+1)} (\pi^{(n)})^T)^{-1} y_c^{(n)}}{(y_f^{(n)})^T (\tilde{\Theta}^{(n+1)})^{-1} y_f^{(n)}}\,.
\end{equation}
Then the following proposition identifies $\hat{g}^{(n)}_{f,i}$ as the direction of the  gradient descent of $\tilde{\rho}(n+1)$ with respect to the parameters $g^{(n)}_{f,i}$.

\begin{Proposition}\label{propkjdhejdh}
It holds true that
\begin{equation}
\tilde{\rho}(n+1)=\rho(n)-\epsilon \sum_{i=1}^{N_f} (g^{(n)}_{f,i})^T \hat{g}^{(n)}_{f,i} +\mathcal{O}(\epsilon^2)\,.
\end{equation}
where the $\hat{g}^{(n)}_{f,i}$ are as in \eqref{eqkjjdjkhedhd}.
\end{Proposition}
\begin{proof}
We deduce from \eqref{eqlkedjdldh}
that, to the first order in $\epsilon$,
\begin{equation}\label{eqkhjeddgekjdh}
\begin{split}
K_{n+1}(x,x')\approx& K_n\big(x,x'\big)+\epsilon \big( (G_{n+1}\circ F_n(x))^T  (\nabla_x K_1)\big( F_n(x),F_n(x')
\\&+ (G_{n+1}\circ F_n(x'))^T  (\nabla_{x'} K_1)\big( F_n(x),F_n(x')\big)\,.
\end{split}
\end{equation}
Therefore, using the symmetry of $K_1$, we have
\begin{equation}
\tilde{\Theta}^{(n+1)}=\Theta^{(n)}+ \epsilon T^{(n)}+\mathcal{O}(\epsilon^2)
\end{equation}
with
\begin{equation}
T^{(n)}_{i,j}=(g^{(n)}_{f,i})^T (\nabla_x K_1)(x_{f,i}^{(n)},x_{f,j}^{(n)})+
(g^{(n)}_{f,j})^T (\nabla_{x} K_1)\big(x_{f,j}^{(n)},x_{f,i}^{(n)})\,.
\end{equation}
We deduce from Proposition \ref{propekjdhgud}  that
\begin{equation}
\rho(\tilde{\Theta}^{(n+1)})=\rho(\Theta^{(n)})-\epsilon \frac{(1-\rho(\Theta^{(n)})) (\hat{y}^{(n)}_f)^T T^{(n)} \hat{y}_f^{(n)}  - (\hat{z}_f^{(n)})^T T^{(n)} \hat{z}_f^{(n)} }{(\hat{y}_f^{(n)})^T \Theta^{(n)} \hat{y}_f^{(n)}}+\mathcal{O}(\epsilon^2)\,,
\end{equation}
which implies the result after simplification.
\end{proof}

The following corollary is a direct application of Prop.~\ref{propkjdhejdh}.
\begin{Corollary}\label{corfkjhfkfhjf}
Let $K_1(x,x')=e^{-\gamma |x-x'|^2}$. It holds true that
\begin{equation}
\rho(\tilde{\Theta}^{(n+1)})=\rho(\Theta^{(n)})-\epsilon \sum_{i=1}^{N_f} (g^{(n)}_i)^T \hat{g}^{(n)}_i +\mathcal{O}(\epsilon^2)\,.
\end{equation}
with $\hat{g}^{(n)}_i:=\frac{4\gamma}{(y_f^{(n)})^T (\Theta^{(n)})^{-1} y_f^{(n)}} \sum_{j=1}^{N_f}   \Gamma^{(n)}_{i,j}x^{(n)}_{f,j} $,
\begin{equation}
\begin{split}
\Gamma^{(n)}_{i,j}=&\delta_{i,j}  \hat{z}^{(n)}_{f,i}    (\Theta^{(n)} \hat{z}^{(n)})_{f,i}
-\hat{z}^{(n)}_{f,i}   \Theta^{(n)}_{i,j}   \hat{z}^{(n)}_{f,j}\\
&- (1-\rho(\Theta^{(n)})) \delta_{i,j} \hat{y}^{(n)}_{f,i}    y_{f,i}^{(n)}+(1-\rho(\Theta^{(n)}))
\hat{y}^{(n)}_{f,i}   \Theta^{(n)}_{i,j}   \hat{y}^{(n)}_{f,j}\,.
\end{split}
\end{equation}
and $\hat{z}^{(n)}_f= (\pi^{(n)})^T (\pi^{(n)}\Theta^{(n)}(\pi^{(n)})^T)^{-1}\pi^{(n)} y_f$, $\hat{y}^{(n)}_f=(\Theta^{(n)})^{-1} y_f^{(n)}$.
\end{Corollary}

\begin{figure}[h]
\begin{center}
\includegraphics[width= 0.5\textwidth]{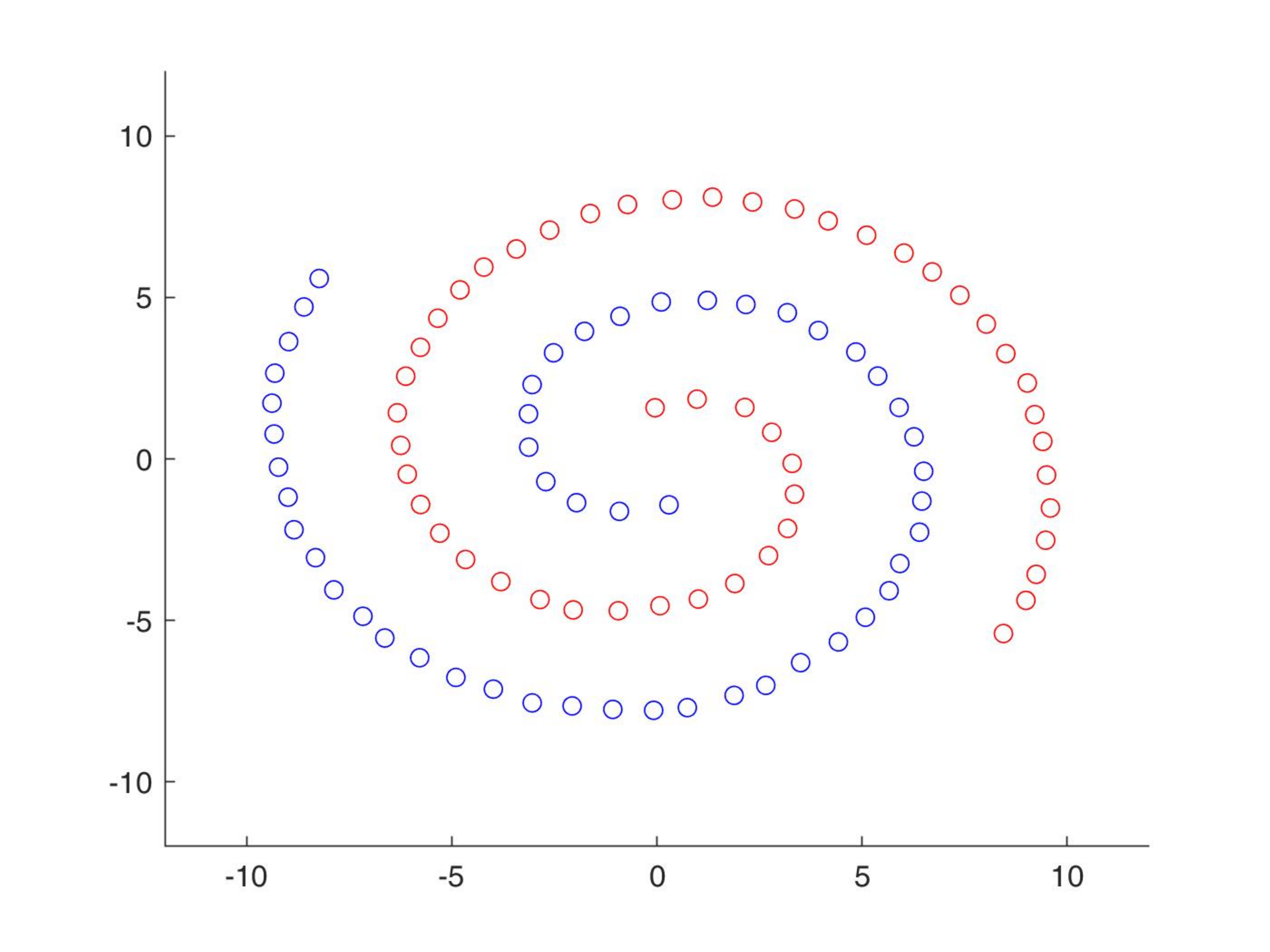}
\caption{Swiss Roll Cheesecake. $N=100$.  Red points have label $-1$ and blue points have label $1$.}
\label{figswissroll1}
\end{center}
\end{figure}

\section{The Flow of the KF algorithm on the Swiss Roll cheesecake}\label{subseccheesecake}

From a numerical analysis  perspective, the flow $F_n(x)$ associated with the Sec.~\ref{secKernelFlow} KF algorithm
 approximates, in the sense of Smoothed Particle Hydrodynamics \cite{gingold1977smoothed}, a flow $F(t,x)$ mapping $\R \times \X$ into $\X$. Writing $x_i$ for the $N$ training data points in $\X$, as $\epsilon\downarrow 0$, $F_n(x_i)$ approximates $X_i(t):=F(t,x_i)$ which (after averaging the effect of the randomized batches) can be identified as the
solution of a system of ODEs  of the form
\begin{equation}\label{eqkjehdghdjh}
\frac{d X}{dt} = \mathcal{G}(X, N, N_f, K_1)\,.
\end{equation}

\subsection{Implementation of the KF algorithm}

We will now explore a few properties of this approximation by  implementing the KF algorithm for the  Swiss Roll Cheesecake illustrated in Fig.~\ref{figswissroll1}.
The dataset is composed of $N=100$ points $x_i$ in $\R^2$ in the shape of two concentric spirals.
 Red points have label $-1$ and blue points have label $1$.
Since our purpose is limited to illustrating properties of the discrete flow $F_n(x)$ associated with the Kernel Flow algorithm we will consider all those points as training points (will not introduce a test dataset) and visualise the trajectories $n\rightarrow F_n(x_i)$  of the data points $x_i$.

The KF algorithm is implemented with the Gaussian kernel of Corollary \ref{corfkjhfkfhjf} with $\gamma^{-1}=4$.
Training is done in random  batches of size $N_f$ and we use $N_c=N_f/2$ to compute the ratio $\rho$ and learn the parameters of the network. We start with $N_f=N$ and as training progresses points of the same color start merging. Therefore to avoid near singular matrices caused by batches with points of the same color sharing nearly identical coordinates, once the distance between two points of the same color is smaller than $10^{-4}$ (in Fig.~\ref{figswissroll2}, \ref{figswissroll3} and \ref{figswissroll4}) we drop one of them out of the set of possible candidates for the batch and decrease $N_f$ by $1$ (the point left out remains advected by $\epsilon G_{n+1}$ but is simply no longer available as an interpolation point defining $G_{n+1}$).

Fig.~\ref{figswissroll2}  shows $F_n(x_i)$ vs $n$ for $1\leq n \leq 7000$.
In Fig.~\ref{figswissroll2} the value of $\epsilon$ is chosen at each step $n$ so that the  perturbation of each data point $x_i$ of the batch is no greater than $p\%$ with $p=10$ for $1\leq n \leq 1000$ and $p=10/\sqrt{n/1000}$ for $n\geq 1000$.
The two spirals quickly unroll into two linearly separable clusters.

\begin{figure}[h]
\begin{center}
\includegraphics[width= \textwidth]{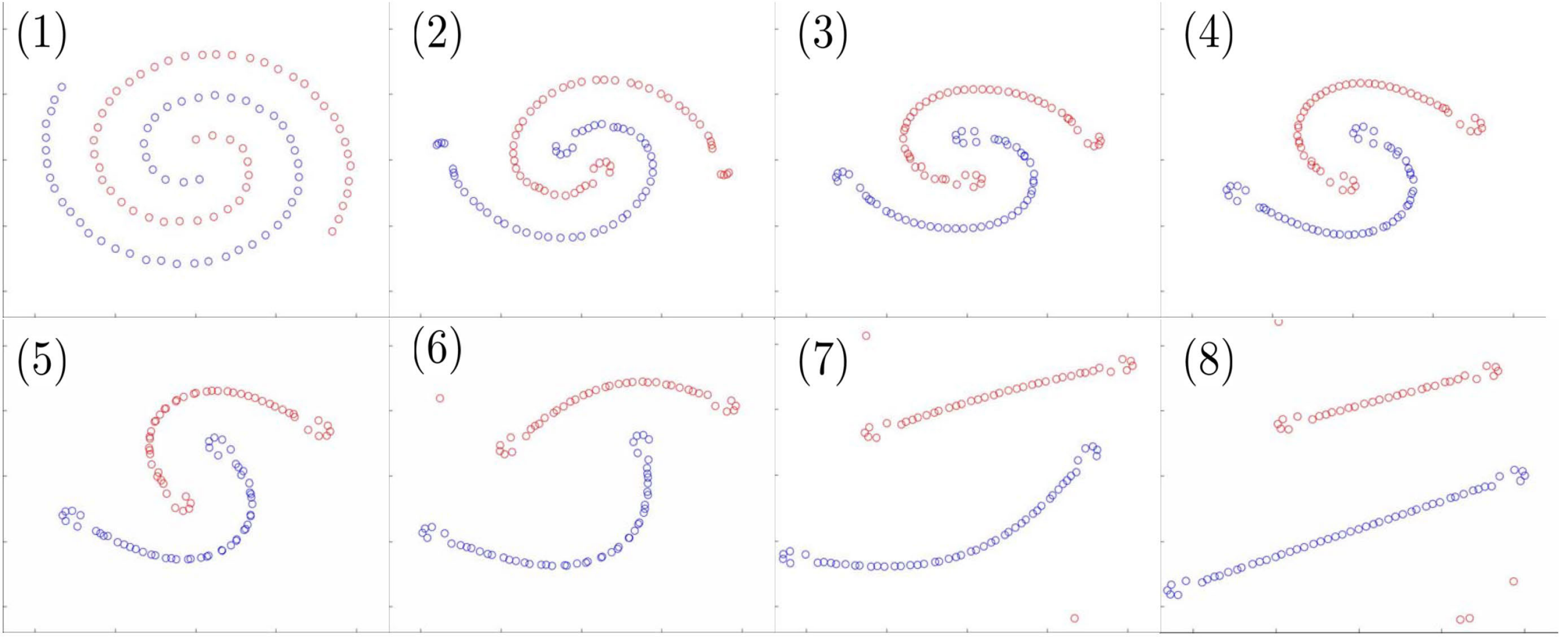}
\caption{$F_n(x_i)$ for 8 different values of $n$. $\epsilon$ is chosen at each step $n$ so that the  perturbation of each data point $x_i$ of the batch is no greater than $0.5\%$.}
\label{figswissroll3}
\end{center}
\end{figure}

Fig.~\ref{figswissroll3}  shows $F_n(x_i)$ vs $n$ for $1\leq n \leq 500000$.
The value of $\epsilon$ is chosen at each step $n$ so that the  perturbation of each data point $x_i$ of the batch is no greater than $0.5\%$. The two intertwined spirals unroll into straight (vibrating) segments.
The final unrolled  configuration appears to be unstable (some red points are ejected out of the unrolled red segment towards the end of the simulation) and although this instability seems to be alleviated  by adjusting  $\epsilon$ to a smaller value at the end of the simulation it seems to also be caused by a combination of  (1)  the stiffness of the flow being simulated (2) the process of permanently removing points from the pool of possible candidates for the batch. Although these points remain advected by the flow and are initially at distance less than $10^{-4}$ of a point of the same color, the stiffness of the flow can quickly increase this distance and separate the point from its group if $\epsilon$ is not small enough.

\begin{figure}[h]
\begin{center}
\includegraphics[width= \textwidth]{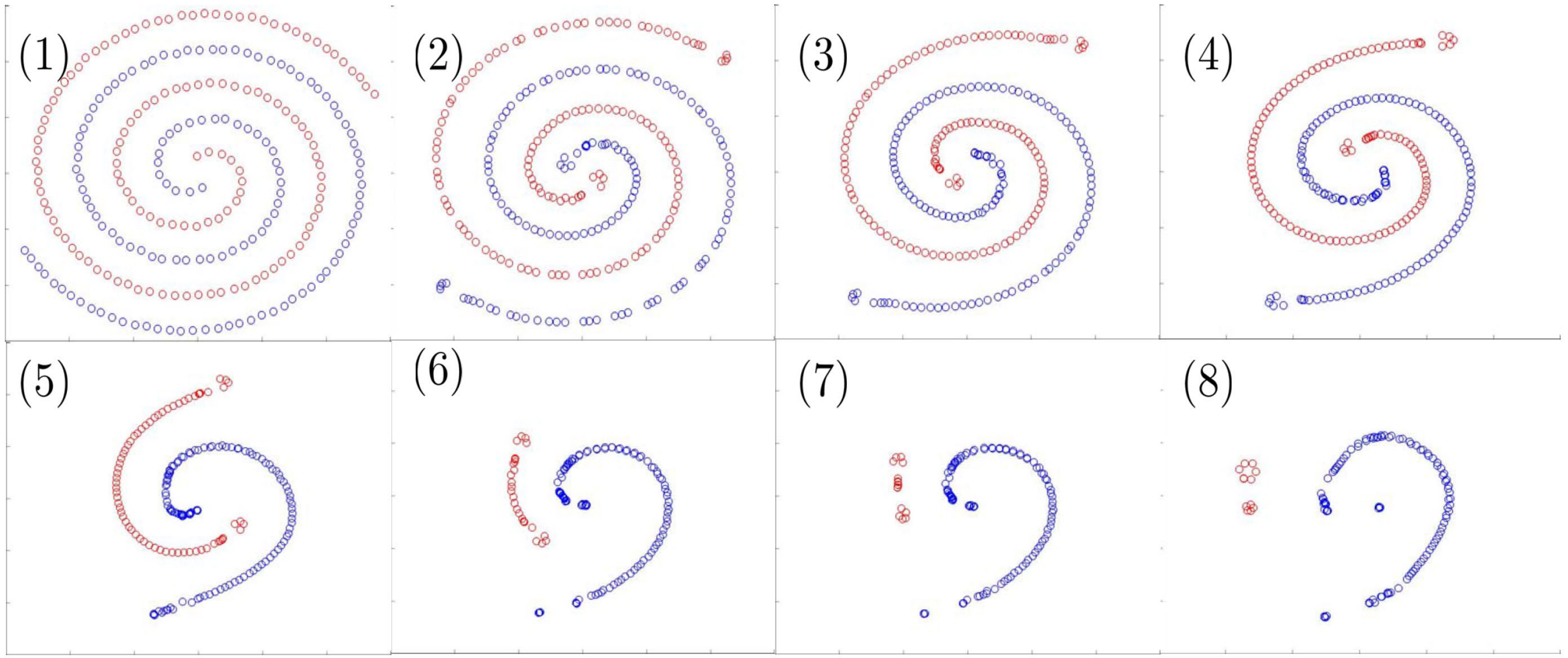}
\caption{$N=250$. $F_n(x_i)$ for 8 different values of $n$. $\epsilon$ is chosen at each step $n$ so that the  absolute perturbation of each data point $x_i$ of the batch is no greater than $0.05$. $K_1(x,x')=e^{-\gamma |x-x'|^2}+\delta(x-x') e^{-\gamma 6^2}$}
\label{figswissroll4}
\end{center}
\end{figure}

\subsection{Addition of nuggets}
The instabilities observed in Fig.~\ref{figswissroll3} seem to vanish (even with significantly larger values of $\epsilon$) after the addition of a nugget (white noise kernel accounting for measurement noise in kriging) to the kernel $K_1$. In
Fig.~\ref{figswissroll4} we consider a longer version of the Swiss Roll Cheesecake ($N=250$). The kernel is
$K_1(x,x')=e^{-\gamma |x-x'|^2}+\delta(x-x') e^{-\gamma 6^2}$ ($\gamma^{-1}=4$).
$\epsilon$ is chosen at each step $n$ so that the  absolute perturbation (maximum translation) of each data point $x_i$ of the batch is no greater than $0.05$.
Fig.~\ref{figswissroll4}  shows $F_n(x_i)$ vs $n$ for $1\leq n \leq 500000$.  The two intertwined spirals unroll into stable
clusters slowly drifting away from each other.

\begin{figure}[!htb]
\begin{center}
\includegraphics[width= \textwidth]{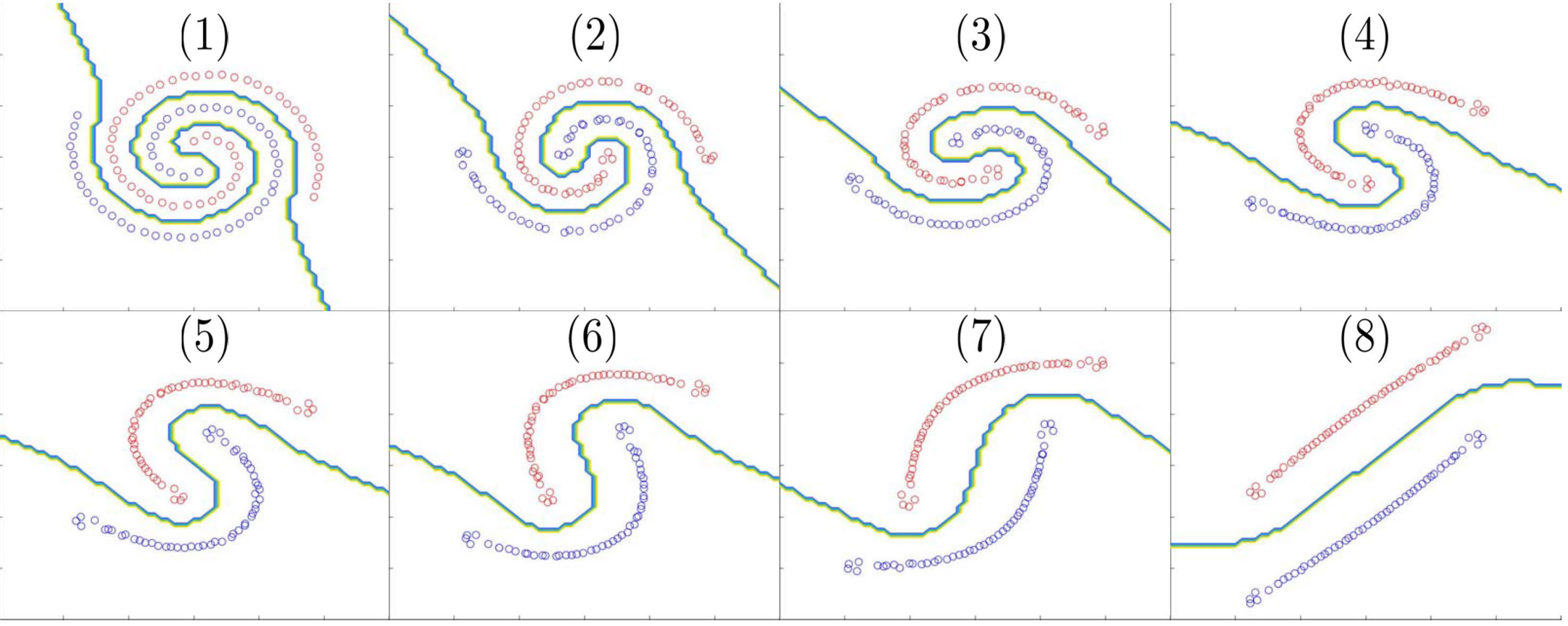}
\caption{$F_n(x_i)$ and decision boundary for 8 different values of $n$.
}
\label{figswissroll5}
\end{center}
\end{figure}

\begin{figure}[!htb]
\begin{center}
\includegraphics[width= \textwidth]{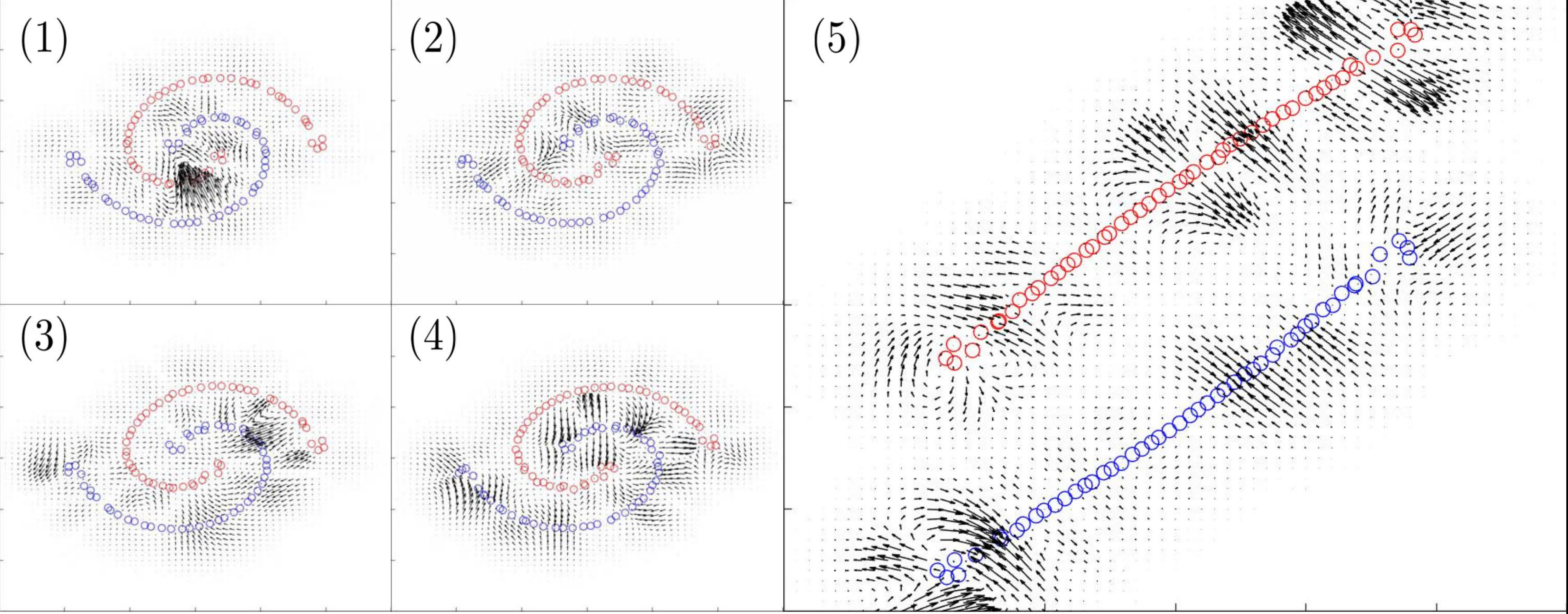}
\caption{ Instantaneous velocity field $F_{n+1}(x)-F_n(x)$. (1-4) show $4$ successive values. (5)  shows the instantaneous velocity field for the final configuration.
}
\label{figswissroll6}
\end{center}
\end{figure}

\subsection{Instantaneous and average vector fields}
With the addition of the nugget, the permanent removal of close points from the pool of candidates is no longer required for
avoiding singular matrices.
In
Fig.~\ref{figswissroll5}, \ref{figswissroll6} and \ref{figswissroll7} (see \cite{youtubevideokf} for videos)   we consider the $N=100$ version of the Swiss Roll Cheesecake. The kernel is
$K_1(x,x')=e^{-\gamma |x-x'|^2}+\delta(x-x') e^{-\gamma 6^2}$ ($\gamma^{-1}=4$).
$\epsilon$ is chosen at each step $n$ so that the  absolute perturbation (maximum translation) of each data point $x_i$ of the batch is no greater than $0.2$. Only a nugget is added and close points are not removed from the pool of candidates (which eliminates the instabilities observed in Fig.~\ref{figswissroll3}).
Fig.~\ref{figswissroll5}  shows $F_n(x_i)$ vs $n$ for $1\leq n \leq 180000$ and the decision boundary between the two classes vs $n$.
 The two intertwined spirals unroll into stable
clusters.
Fig.~\ref{figswissroll6} shows  the instantaneous velocity field $F_{n+1}(x)-F_n(x)$.
Fig.~\ref{figswissroll7} shows  the average velocity field $10(F_{n+300}(x)-F_n(x))/300$.
The difference between the instantaneous and average velocity fields is and indication of the presence of multiple time scales caused by the randomization process and the stiffness of the underlying flow.

\begin{figure}[h]
\begin{center}
\includegraphics[width= \textwidth]{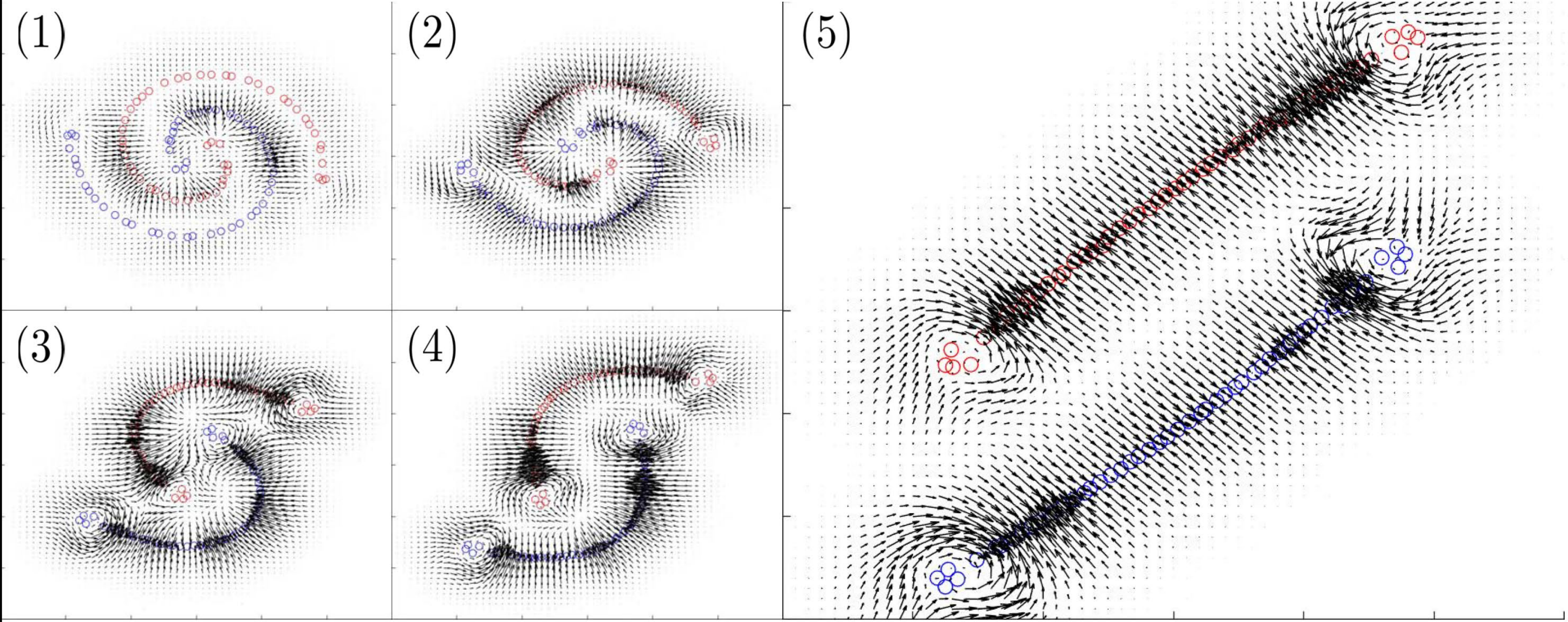}
\caption{Average velocity field $10(F_{n+300}(x)-F_n(x))/300$ for 5 different values of $n$.
}
\label{figswissroll7}
\end{center}
\end{figure}

\subsection{The continuous flow.}
 The intriguing behavior of these flows, calls for their investigation from the
 a numerical integration perspective (note that an ODE perspective emerges as $\epsilon \downarrow 0$, an SDE perspective is relevant when $\epsilon$ is non-null due to the randomization of the batches, a PDE perspective emerges as $\epsilon \downarrow 0$ and $N\rightarrow \infty$, and an SPDE approximation perspective is relevant when $\epsilon$ is non-null and $N$ is finite).

Note that the KF flow $F_n(x)$ can be seen (as $\epsilon \downarrow 0$) as the numerical approximation of a continuous flow $F(t,x)$.
identified as the solution of the dynamical system
\begin{equation}\label{eqkrhuhgtdjh0}
\frac{\partial F(t,x)}{\partial t} = -  \E_{X, \pi}\Bigg[ \Big(\big(\nabla_Z \rho(X,Z,\pi)\big)^T
\big(K_1(Z,Z)\big)^{-1} K_1\big(Z,x\big)\Big) \Bigg\rvert_{Z=F(X,t)}\Bigg]\,,
\end{equation}
with initial condition $F(0,x)=x$ and where the elements of \eqref{eqkrhuhgtdjh0} are defined as follows.
$X$ is a random vector of $\X^{N_f}$ representing the random sampling of the training data in a batch size $N_f$.
Writing $u(X)\in \Y^{N_f}$ for the vector whose entries are the labels of the entries of $X\in \X^{N_f}$ and
$\pi$ for a random $N_c\times N_f$ matrix corresponding to the selection of $N_c$ elements at out $N_f$ (at random, uniformly, without replacement), $\rho$, in \eqref{eqkrhuhgtdjh0}, is defined as follows
\begin{equation}\label{eqkertdjh0}
\rho(X,Z,\pi) =  1- \frac{ u( X)^T \pi^T (K_1(\pi Z, \pi Z))^{-1}\pi u(X) }{u(X)^T (K_1(Z,Z))^{-1}u(X)}\,.
\end{equation}
The average vector field  in Fig.~\ref{figswissroll7} is an approximation of the right hand side of \eqref{eqkrhuhgtdjh0}
and the convergence of $F_{\textrm{round}(t/\epsilon)}(x)$ towards $F(t,x)$ as $\downarrow 0$ is in the sense of two-scale flow convergence described in  \cite{tao2010nonintrusive}.

   \begin{table}[h]
 \begin{center}
    \begin{tabular}{ | l | l | l | l | l|}
    \hline
    $N_I$ & Average error & Min error & Max error & Standard Deviation \\ \hline
    $6000$ &  $0.014$   &   $0.0136 $   &   $0.0143 $  &   $1.44 \times 10^{-4}$\\ \hline
    $600$ &  $0.014$    &   $0.0137$   &   $0.0142 $ &    $9.79\times 10^{-5}$  \\ \hline
    $60$ & $0.0141$ &      $0.0136$      & $0.0146$   &  $2.03\times 10^{-4}$ \\
    \hline
    $10$ &   $0.015$ &      $0.0136$   &   $0.0177$  &   $7.13\times 10^{-4}$\\ \hline
    \end{tabular}
    \caption{MNIST test errors using $N_I$ interpolation points}\label{table1}
\end{center}
\end{table}

\begin{figure}[h]
\begin{center}
\includegraphics[width= \textwidth]{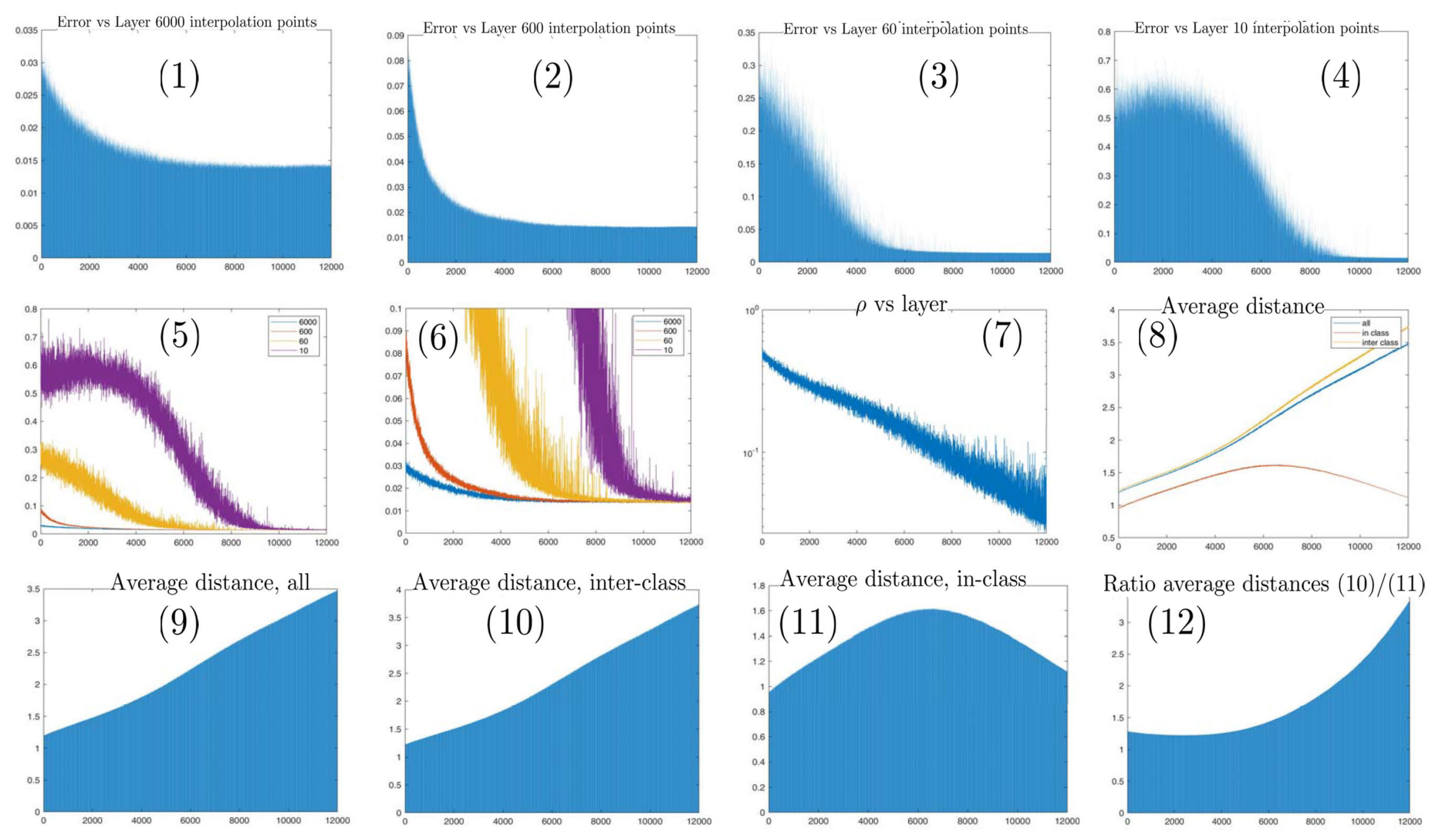}
\caption{Results for MNIST. $N=60000$, $N_f=600$ and $N_c=300$. (1) Test error vs depth $n$ with $N_I=6000$ (2) Test error vs depth $n$ with $N_I=600$ (3) Test error vs depth $n$ with $N_I=60$ (4) Test error vs depth $n$ with $N_I=10$  (5,6) Test error vs depth $n$ with $N_I=6000, 600, 60, 10$ (7) $\rho$  vs depth $n$ (8)  Mean-squared distances   between  images $F_n(x_i)$ (all, inter class and in class) vs depth $n$ (9)
 Mean-squared distances   between  images (all) vs depth $n$  (10) Mean-squared distances between  images (inter class) vs depth $n$  (11) Mean-squared distances between  images (in class) vs depth $n$ (12) Ratio $(10)/(11)$. }
\label{figmnist1}
\end{center}
\end{figure}
\begin{figure}[h]
\begin{center}
\includegraphics[width= \textwidth]{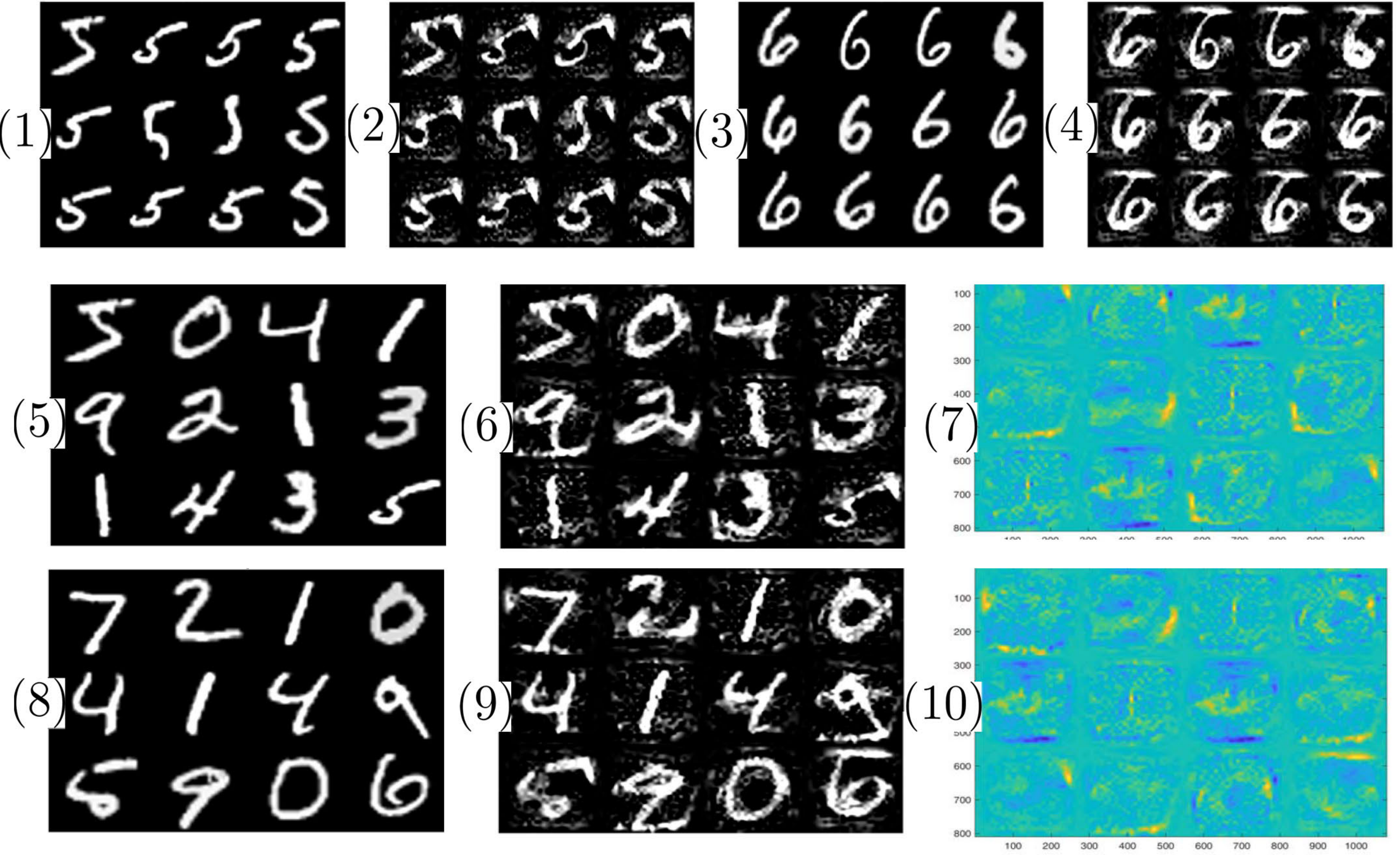}
\caption{Results for MNIST. $N=60000$, $N_f=600$ and $N_c=300$. (1, 3, 5) Training data $x_i$ (2, 4, 6) $F_n(x_i)$ for $n=12000$ (7) $F_n(x_i)-x_i$ for training data and $n=12000$ (8) Test data $x_i$ (9) $F_n(x_i)$ for test data and $n=12000$ (10) $F_n(x_i)-x_i$ for test data and $n=12000$. }
\label{figmnist2}
\end{center}
\end{figure}

\section{Numerical experiments with the MNIST dataset}\label{kfmnist}

We will now implement, test and analyze the Sec.~\ref{secKernelFlow} KF algorithm on the MNIST dataset \cite{yann1998mnist}.
This training set is composed of $60000$, $28\times 28$ images of handwritten digits (partitioned into $10$ classes)  with a corresponding vector of $60000$ labels (with values in $\{1,\ldots,9, 0\}$).
The test set is composed of $10000$, $28\times 28$ images of handwritten digits with a corresponding vector of $10000$ labels.

\subsection{Learning with small random batches of the full training dataset.}\label{subsecwdgejkhdgjd}
We first implement the Sec.~\ref{secKernelFlow} KF algorithm  with the full training set, i.e. $N=60000$. Images are normalized to have $L^2$ norm one and we use the Gaussian kernel of Corollary \ref{corfkjhfkfhjf}  and set $\gamma^{-1}$ equal to the mean squared distance between training images.
Training is performed in random batches of size $N_f=600$ and we use $N_c=300$ to compute the ratio $\rho$ and learn the parameters of the network (we do not use a nugget and we do not exclude points that are too close from those batches). The value of $\epsilon$ is chosen at each step $n$ so that the  perturbation of each data point $x_i$ of the batch
is no greater than $1\%$ ($\epsilon=0.01\times \max_i \frac{ |x^{(n)}_{f,i}|_{L^2}}{|\hat{g}^{(n)}_{f,i}|_{L^2}}$).

 Classification of the test data is performed by interpolating a subset of $N_I$ images/labels $(x_i,y_i)$ of the training data
 with the kernel $K_n$ (the kernel at step/layer $n$ in the  Sec.~\ref{refdklejdhhdkj} Kernel Flow algorithm).
 Here each $x_i$ is a $28\times 28$ image and each $y_i$ is a unit vector $e_j$ in $\R^{10}$ pointing in the direction of the class of the images (e.g. $y_i=(1,0,0,0,0,0,0,0,0,0)=e_1$ if the class of image $x_i$ is $j=1$ and $y_i=(0,0,0,0,0,0,0,0,0,1)=e_{10}$ if the class of image $x_i$ is $j=0$).
 The interpolant $u_n$ is a function from $\R^{28\times 28}$ to $\R^{10}$ and the class of an image $x$ is simply identified as $\operatorname{argmax}_{j} e_j^T \cdot u_{n}(x)$.
 The Sec.~\ref{refdklejdhhdkj} Kernel Flow algorithm is implemented (using Remark \ref{rmkksjhddg7}) with $N=60000$, $N_f=600$ and $N_c=300$ and ended for $n=12000$ (resulting in a network with $12000$ layers).

  Table \ref{table1} shows
  test errors (with the  test data composed of $10000$ images) obtained with $N_I=6000, 600, 60$ and $10$ interpolation points (selected at random uniformly without replacement and conditioned on containing an equal number of example from each class to avoid degeneracy for $N_I=10, 60$). The second column shows  errors averaged over the last  $100$ layers of the network (i.e. obtained by interpolating the $N_I$ data points with $K_n$ for $n=11901, 11902, \ldots, 12000$). The third, fourth and fifth columns show the min, max and standard deviation of the error over the same last  $100$ layers of the network. Surprisingly, around layer $n=12000$, the kernel $K_n$ achieves an average error of about $1.5\%$ with only $10$ data points (by using only 1 random example of each digit as an interpolation point).
  Multiple runs of the algorithm suggest that those results are stable.

Fig.~\ref{figmnist1} shows test errors vs depth $n$ (with $N_I=6000, 600, 60, 10$ interpolation points), the value of the ratio $\rho$ vs $n$ (computed with $N_f=600$ and $N_c=300$) and the mean squared distances between (all, inter class and in class) images $F_n(x_i)$ vs $n$.
Observe that all  mean-squared distances increase until $n\approx 7000$. After $n\approx 7000$ the in class mean-squared distances  decreases with $n$ whereas the inter-class  mean-squared distances  continue increasing. This suggests that after $n\approx 7000$ the algorithm starts clustering the data per class. Note also that while the test errors,  with $N_I=6000,600$ interpolation points, decrease immediately and sharply, the test errors with $N_I=10$ interpolation points increase slightly until $n\approx 3000$ towards $60\%$, after which they drop and seem to stabilize around $1.5\%$ towards $n\approx 10000$.

It is known that iterated random functions typically converge because they are contractive on average  \cite{diaconis1999iterated,dunlop2017deep}. Here training appears to create  iterated functions that are contractive with each class but expansive between classes.

Fig.~\ref{figmnist2} shows $10\times x_i$ and $10\times F_n(x_i)$ and $20\times (F_n(x_i)-x_i)$ for $n=12000$, training images and test images.
The algorithm appears to introduce small, archetypical, and class dependent, perturbations in those images.

\begin{figure}[h]
\begin{center}
\includegraphics[width= \textwidth]{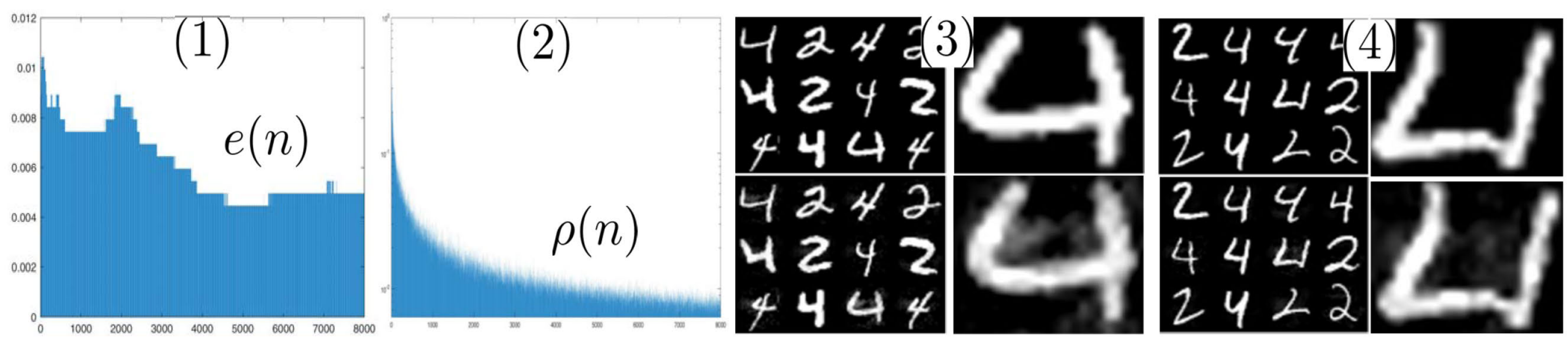}
\caption{(1) Test error vs $n$ (2) $\rho$ vs $n$ (3) $x_i$ and $F_n(x_i)$ for $n=8000$ and $i$ corresponding to the first 12 training images (4)  $x_i$ and $F_n(x_i)$ for $n=8000$ and $i$ corresponding to the first 12 test images. MNIST with classes $2$ and $4$, $600$ training images and $100$ test images.}
\label{fig2and4}
\end{center}
\end{figure}
\begin{figure}[h]
\begin{center}
\includegraphics[width= \textwidth]{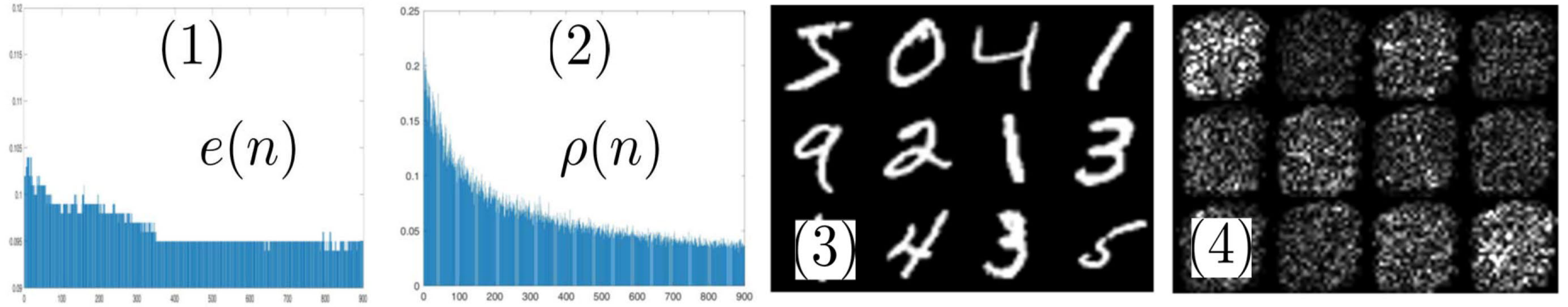}
\caption{MNIST with all classes, $1200$ training images and $2000$ test images. (1) Test error vs $n$ (2) $\rho$ vs $n$ (3)
Training images $x_i$ (4) $10\times \operatorname{abs}(F_n(x_i)-x_i)$ for $n=900$   }
\label{fullmnist}
\end{center}
\end{figure}
\begin{figure}[h]
\begin{center}
\includegraphics[width= \textwidth]{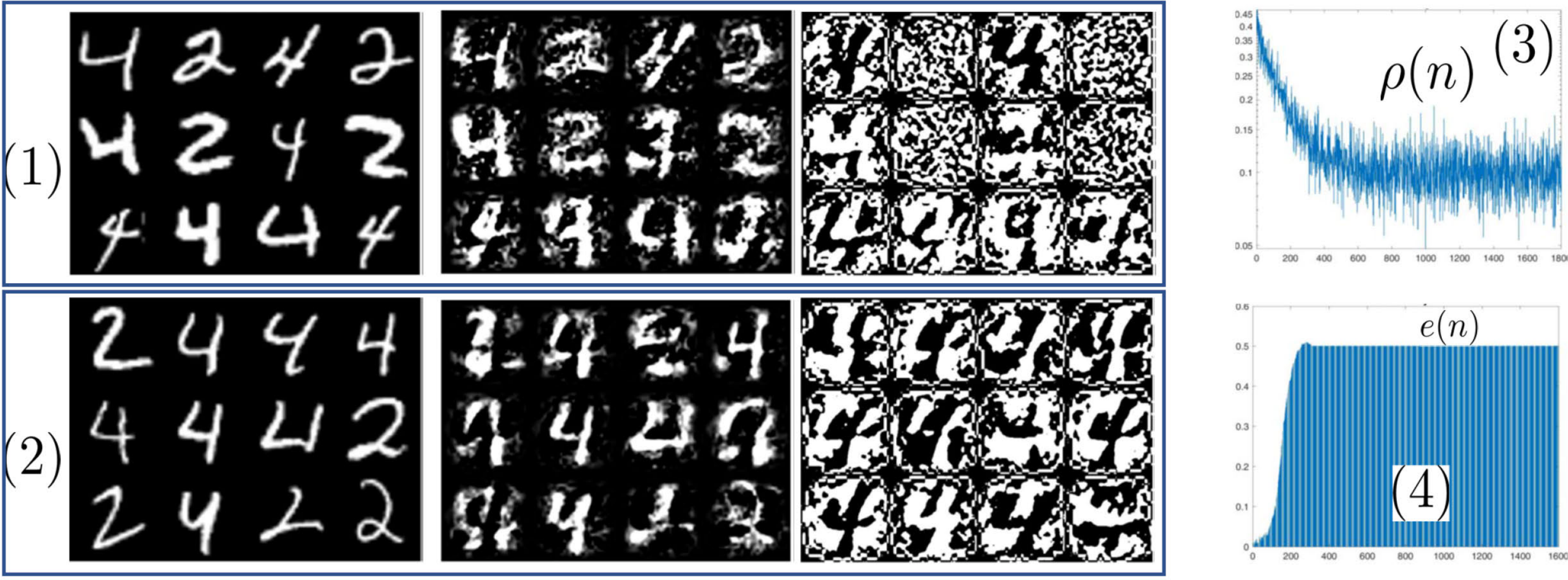}
\caption{MNIST with classes $2$ and $4$, $600$ training images and $100$ test images. Step sizes are too large, $\rho$ decreases but the modes collapse. (1) $x_i$ and $F_n(x_i)$ for $n=1,100, 1800$ and $i$ corresponding to the first 12 training images
(2) $x_i$ and $F_n(x_i)$ for $n=1, 100, 1800$ and $i$ corresponding to the first 12 test images (3) $\rho$ vs $n$ (4) Test error  vs $n$}
\label{fig2and4modecollapse}
\end{center}
\end{figure}

\subsection{Bootstrapping, Brittleness and Data Archetypes}\label{subsecboot}
In Sec.~\ref{subsecwdgejkhdgjd} we used the whole training set of $N=60000$ images to train our network and Fig.~\ref{figmnist1}.
Although the accuracy the network increases (between $n=1$ and $n=12000$) when using only $N_I=6000, 600, 60, 10$ interpolation points, the accuracy of the network with $N_I=60000$ (the full training set as interpolation points) does not seem to be significantly impacted  by the training (the error with $N_I=60000$ is $0.0128$ at $n=1$ and $n=0.0144$ at $n=12000$).
In that sense, all the  algorithm seems to do is to transfer the information contained in the $60000$ data-points to the kernel $K_n$.
Can the accuracy of the interpolation with the full training set be improved? Can the algorithm extract information that cannot already be extracted by performing a simple interpolation (with a simple Gaussian kernel) with the full training dataset?
To answer these questions we will now implement the  Sec.~\ref{refdklejdhhdkj} Kernel Flow algorithm with subsets of  training and test images and train the network with $N_f=N$ (the size of each batch is equal to the total number of training images), $N_c=N_f/2$ and possibly subsets of the set of all classes. We work with raw images
(not normalized to have $L^2$ norm one). We use the Gaussian kernel of Corollary \ref{corfkjhfkfhjf} and identify $\gamma^{-1}$ as the mean squared distance between training images.
We first take $N=600$ ($600$ training images) showing only twos and fours and attempt to classify $100$ test images (with only twos and fours). Fig.~\ref{fig2and4} corresponds to a successful outcome (with small adapted step sizes $\epsilon$) and shows  the test error vs depth $n$, $\rho$ vs $n$ and
$x_i$ and $F_n(x_i)$ for $n=8000$.  These illustrations suggest that the network can bootstrap data and improve accuracy by introducing small (nearly imperceptible to the naked eye) perturbations to the dataset.
Fig.~\ref{fullmnist} shows another successful run with $N=1200$ training images, $1200$ test images and the full set of $10$ classes.

\paragraph{Mode collapse from going too deep, too fast, with $N_f=N$.}
Fig.~\ref{fig2and4modecollapse} shows a failed outcome (with very large step sizes $\epsilon$, the other parameters are the same as in Fig.~\ref{fig2and4}). Although the ratio $\rho$ decreases during training the error blows up towards $50\%$ and the $F_n(x_i)$ seem to collapse towards two images: a four and a random blur.
We will explain and analyze this mode collapse below.

\subsection{Mode collapse, brittleness of deep learning}
The mode collapse observed in Fig.~\ref{fig2and4modecollapse} is interesting for several reasons. First it shows that a decreasing $\rho$ is not sufficient to ensure generalization and learning. Indeed writing $w$ for a function exactly interpolating (fitting) the training data  the kernel $K_w(x,y)=w(x)w(y)$ would lead to a perfect fit (and hence a value $\rho=0$) of the training set with any number of interpolation points (and in particular one). Although this $K_w$ is positive but degenerate, if the space of kernels explored by the algorithm is large enough, then, unless $N_f\ll N$, it is not clear what would prevent the algorithm from over-fitting and converging towards those pathological kernels.

The brittleness of deep learning \cite{szegedy2013intriguing} is a well known phenomenon predicted \cite{MMckerns2013} from the
 brittleness of doing inference in large dimensional spaces \cite{OSS:2013, OwhadiScovel:2013, owhadiBayesiansirev2013, owhadi2017qualitative}. The mechanisms at play in \cite{owhadiBayesiansirev2013}  suggest that those instabilities may be unavoidable if the inference space is too large and could to some degree be alleviated through a compromise between accuracy and robustness \cite{owhadi2017qualitative}.
 From that perspective the learning of the Green's function in Sec.~\ref{secjhdkhj33} appears to be stable because of strong constraints imposed on the space of kernels (by the structure of the underlying PDE). In Sec.~\ref{subsecwdgejkhdgjd} the difference in size between the training dataset ($N$) and that of the batches ($N_f$) seems to have a stabilizing effect on the algorithm.
 The pathologies observed in Fig.~\ref{fig2and4modecollapse}, and mechanisms leading to brittleness \cite{OSS:2013, OwhadiScovel:2013, owhadiBayesiansirev2013, owhadi2017qualitative, miller2018robust}  seem to suggest that will small data sets and a large space of admissible kernels instabilities may occur and could be alleviated by introducing further constraints on the space of kernels.

 \subsection{Classification archetypes}
 The brittleness of deep learning \cite{OSS:2013,  szegedy2013intriguing} has lead to the construction of libraries of adversarial examples
 whose persistence in the physical work  \cite{kurakin2016adversarial} be exploited by an adversary   \cite{huang2011adversarial}.
 These adversarial examples are constructed via small (near-undetectable to the naked eye) data-dependent (non-random) perturbations of the original images.
 The Kernel Flow algorithm seems to exploit the brittleness of deep learning in the opposite direction (towards improved performance), i.e.
 as suggested in Fig.~\ref{figmnist2} and \ref{fig2and4} the Kernel Flow algorithm seems to improve performance through the construction of residual maps introducing  small (data-dependent) perturbations to the original images.
 The resulting images $F_n(x_i)$ at profound depths could be interpreted as archetypes of the  classes being learned.

\begin{figure}[h]
\begin{center}
\includegraphics[width= \textwidth]{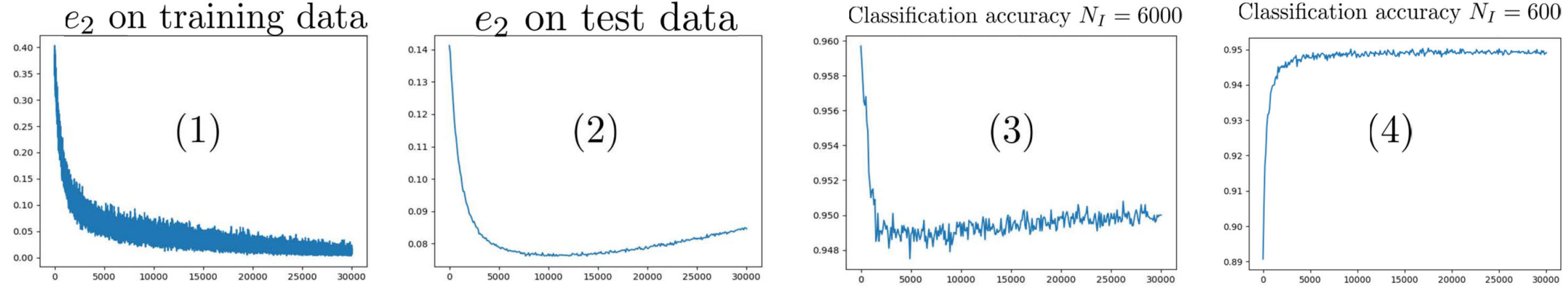}
\caption{Minimizing mean squared interpolation error rather than $\rho$ may not lead to generalization for KF.
(1) Mean squared interpolation error $e_2(n)$ calculated with random subsets of the training data ($N_f=600$ and $N_c=300$)
 (2) Mean squared interpolation error $e_2(n)$ calculated with the test data ($N_I=6000$)
 (3) Classification accuracy (using $N_I=6000$ interpolation points and all $10000$ test data points) vs $n$ (4) Classification accuracy (using $N_I=600$ interpolation points and all $10000$ test data points) vs $n$.}
\label{figmine2}
\end{center}
\end{figure}

 \subsection{On generalization}
 Why the KF algorithm does not seem to overfit the data? Why is it capable of generalization?
  From an initial perspective the KF algorithm appears to promote generalization by grouping data points into clusters according to their classes. However the reason for its generalization properties appears to be more subtle and defining $\rho$ through the RKHS norm seems to also play a role (minimizing $\rho$ by aligning the eigensubpace corresponding to the lowest eigenvalues of the kernel with the training data.).

  Indeed, using the notations of Sec.~\ref{secKernelFlow}, let $v(x_{f,i}^{(n)})$ be the predicted labels of the $N_f$ points $x_{f,i}^{(n)}$  obtained by interpolating a random subset of $N_c=N_f/2$ points $(x_i,y_i)$ with the kernel $K_1$ and write
  $e_2:=\sum_{i=1}^{N_f} |y_{f,i}^{(n)}- v(x_{f,i}^{(n)})|^2 $ for the mean squared error between training labels and predicted labels.
  Then minimizing $e_2$ instead of $\rho$ may lead to a decreasing test classification accuracy rather than an increasing one as shown in
  Fig.~\ref{figmine2} (using the MNIST dataset with $N=60000$ training points, $10000$ test points, $N_f=600$, $N_c=300$ for the random batches and $N_I=600, 6000$ interpolation points for calculating classification accuracies). Note that although the mean squared interpolation error decreases for the training and the test data, the classification error (on the 10000 test data points of MNIST using 6000 interpolation points) increases.

\begin{table}[h]
 \begin{center}
    \begin{tabular}{ | l | l | l | l | l|}
    \hline
    $N_I$ & Average error & Min error & Max error & Standard Deviation \\ \hline
    $6000$ &   $0.0969   $&$    0.0944  $&$    0.1   $&$   7.56   \times 10^{-4}$\\ \hline
    $600$ &   $0.0977    $&$  0.0951   $&$   0.101  $&$   8.57  \times 10^{-4}$   \\ \hline
    $60$ & $0.114  $&$     0.0958  $&$    0.22   $&$     0.0169   $  \\
    \hline
    $10$ &   $0.444  $&$    0.15   $&$    0.722   $&$    0.096  $\\ \hline
    \end{tabular}
    \caption{Fashion-MNIST test errors (between layers $15000$ and $25000$) using $N_I$ interpolation points}\label{tablef1}
\end{center}
\end{table}

     \begin{table}[h]
 \begin{center}
    \begin{tabular}{ | l | l | l | l | l|}
    \hline
    $N_I$ & Average error & Min error & Max error & Standard Deviation \\ \hline
    $6000$ &   $0.10023 $ & $     0.0999  $&$    0.1006  $&$   1.6316\times 10^{-4}$\\ \hline
    $600$ &   $0.10013  $&$    0.0999   $&$   0.1004   $&$   1.1671\times 10^{-4}$   \\ \hline
    $60$ & $0.10018 $&$     0.0999  $&$    0.1005  $&$   1.445    \times 10^{-4}$  \\
    \hline
    $10$ &   $0.10018  $&$     0.0996 $&$     0.1009 $&$    2.2941 \times 10^{-4}$\\ \hline
    \end{tabular}
    \caption{Fashion-MNIST test errors (between layers $49901$ and $50000$) using $N_I$ interpolation points}\label{tablef2}
\end{center}
\end{table}

\begin{figure}[h]
\begin{center}
\includegraphics[width= \textwidth]{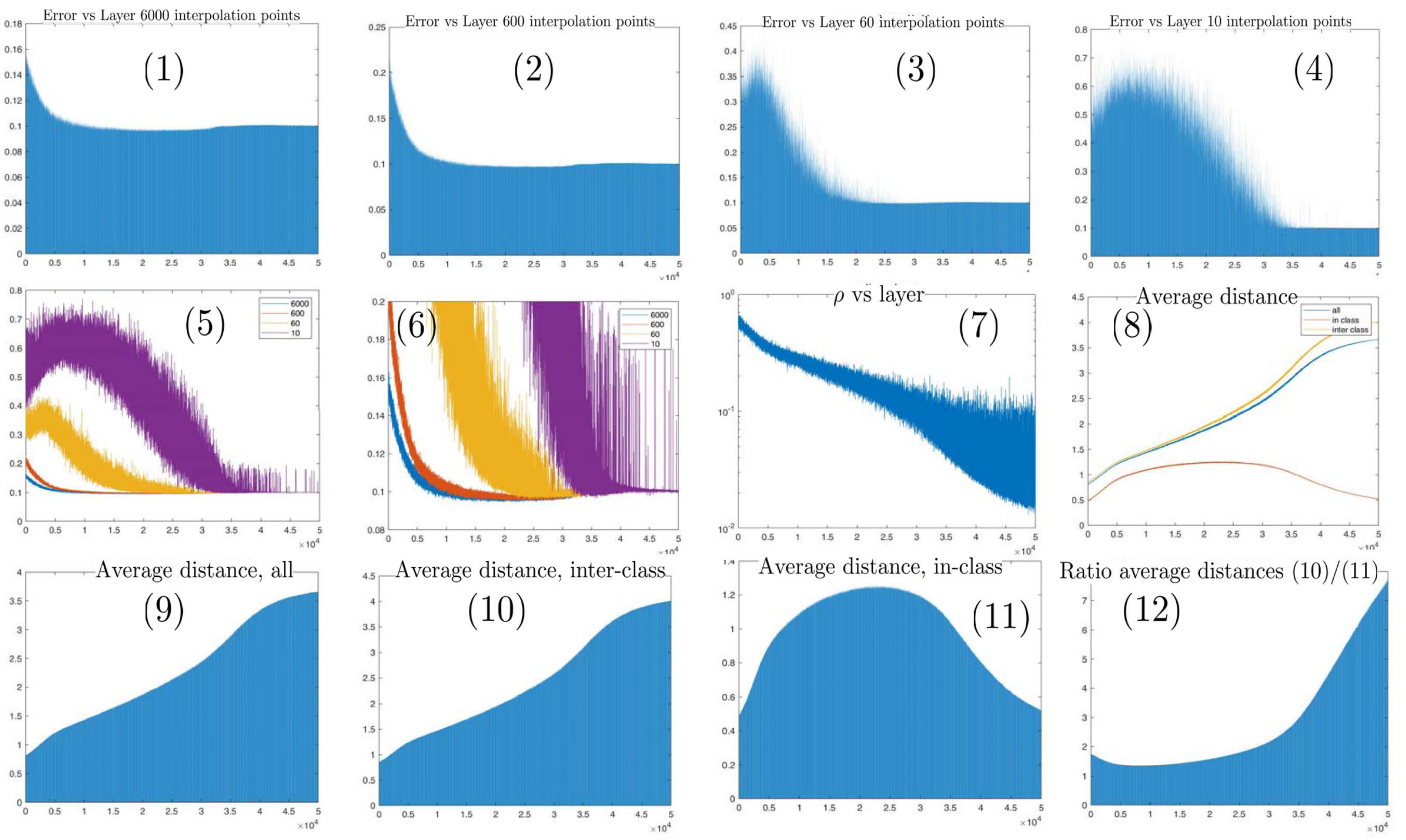}
\caption{Results for Fashion-MNIST. $N=60000$, $N_f=600$ and $N_c=300$. (1) Test error vs depth $n$ with $N_I=6000$ (2) Test error vs depth $n$ with $N_I=600$ (3) Test error vs depth $n$ with $N_I=60$ (4) Test error vs depth $n$ with $N_I=10$  (5,6) Test error vs depth $n$ with $N_I=6000, 600, 60, 10$ (7) $\rho$  vs depth $n$ (8)  Mean-squared distances   between  images $F_n(x_i)$ (all, inter class and in class) vs depth $n$ (9)
 Mean-squared distances   between  images (all) vs depth $n$  (10) Mean-squared distances between  images (inter class) vs depth $n$  (11) Mean-squared distances between  images (in class) vs depth $n$ (12) Ratio $(10)/(11)$. }
\label{figfmnist0}
\end{center}
\end{figure}

\begin{figure}[h]
\begin{center}
\includegraphics[width= \textwidth]{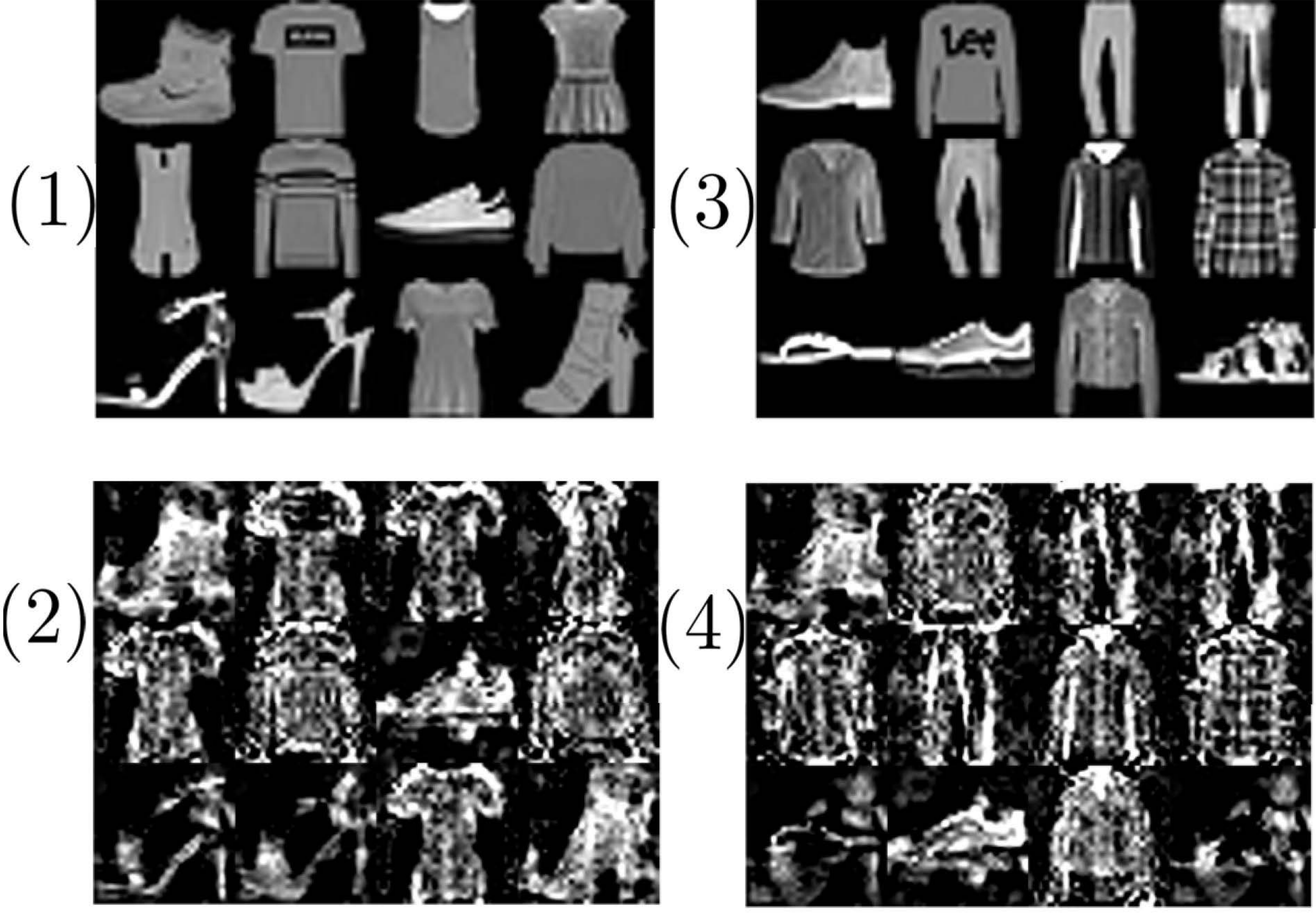}
\caption{Results for Fashion-MNIST. $N=60000$, $N_f=600$ and $N_c=300$. (1) Training data $x_i$ (2) $F_n(x_i)$ training data and $n=50000$ (3) Test data $x_i$ (9) $F_n(x_i)$ for test data and $n=50000$. }
\label{figfmnist1}
\end{center}
\end{figure}

\section{Numerical experiments with the Fashion-MNIST dataset}
We now implement and test  the Sec.~\ref{secKernelFlow} KF algorithm with the Fashion-MNIST dataset \cite{xiao2017fashion}.
As with the  MNIST dataset \cite{yann1998mnist}, the Fashion-MNIST dataset is composed of $60000$, $28\times 28$ images portioned into $10$ classes (T-shirt/top, trouser, pullover, dress, coat, sandal, shirt, sneaker, bag, ankle boot)
with a corresponding vector of $60000$ labels (with values in $\{0,\ldots,9\}$).
The test set is composed of $10000$, $28\times 28$ images of handwritten digits with a corresponding vector of $10000$ labels.
\subsection{Network trained to depth $n=50000$}\label{subsec82782}
The KF algorithm is implemented with the exact same parameters as for the MNIST dataset (Sec.~\ref{subsecwdgejkhdgjd}, in particular it does not require any manual tuning of hyperparameters nor a laborious process of guessing an architecture for the network).
In particular, images are normalized to have $L^2$ norm one we use the Gaussian kernel of Corollary \ref{corfkjhfkfhjf} and set $\gamma^{-1}$ equal to the mean squared distance between training images.
Training is performed in random batches of size $N_f=600$ and we use $N_c=300$ to compute the ratio $\rho$ and learn the parameters of the network (we do not use a nugget and we do not exclude points that are too close from those batches). The value of $\epsilon$ is chosen at each step $n$ so that the  perturbation of each data point $x_i$ of the batch
is no greater than $1\%$ ($\epsilon=0.01\times \max_i \frac{ |x^{(n)}_{f,i}|_{L^2}}{|\hat{g}^{(n)}_{f,i}|_{L^2}}$).

The network is trained to depth $n=50000$. Table \ref{tablef1} shows test error statistics (on the full test dataset) using the kernel $K_n$ for $15000 \leq n\leq 25000$ and $N_I=6000, 600, 60, 10$ interpolation points. Table \ref{tablef2} shows test error statistics (on the full test dataset) using the kernel $K_n$ for $49901 \leq n\leq 50000$ and $N_I=6000, 600, 60, 10$ interpolation points.
Fig.~\ref{figfmnist0} plots test errors vs $n$ using $N_I=6000, 600, 60, 10$ interpolation points and shows average distances between $F_n(x_i)$ vs $n$ (for $x_i$ selected uniformly at random amongst all training images, within the same or in different classes).

Note that although the network  achieves an average test error of $9.7\%$ between layers $15000$ and $25000$ with $N_I=600$, average test errors for $N_I=60, 10$ interpolation points require a depth of more than $37000$ layers to achieve comparable accuracies. Note that the average error around layer $50000$ with $N_I=10$ interpolation points is $10\%$ and does not seem to significantly depend on $N_I$.
The average error ($\approx 9.7\%$) of the classifier with $N_I=600, 6000$ interpolation points between layers $15000$ and $25000$
and the
slight increase of average test errors with $N_I=600, 6000$ interpolation points between layers $25000$ and $50000$ (from  $\approx 9.7\%$ to  $\approx 10\%$) seem to decrease with the value of $\epsilon$. \eqref{eqklejdhkdj} could be interpreted an underlying stochastic differential equation with an explicit scheme with time steps $\epsilon$ and the efficiency of the resulting classifier seems to improve as $\epsilon \downarrow 0$.

Note from Fig.~\ref{figfmnist0}.(8-12), \ref{figfmnist1} and  \ref{figfmnist2} that $F_n(x_i)$  converges towards an archetype of the class of $x_i$ and that (after layer $n\approx 25000$) the KF algorithm  contracts distances within each class while continuing to expand distances between different classes.

Interpolation with $K_1$ and all $N_I=N=60000$ training points used as interpolation points results in $12.75\%$ test error and interpolation with
$K_n$ with $n=50000$ and all $N_I=N=60000$ training points used as interpolation points results in $10\%$ test error. Therefore the KF algorithm appears to bootstrap information contained in the training data in the sense discussed in Sec.~\ref{subsecboot}.
\begin{figure}[h]
\begin{center}
\includegraphics[width= \textwidth]{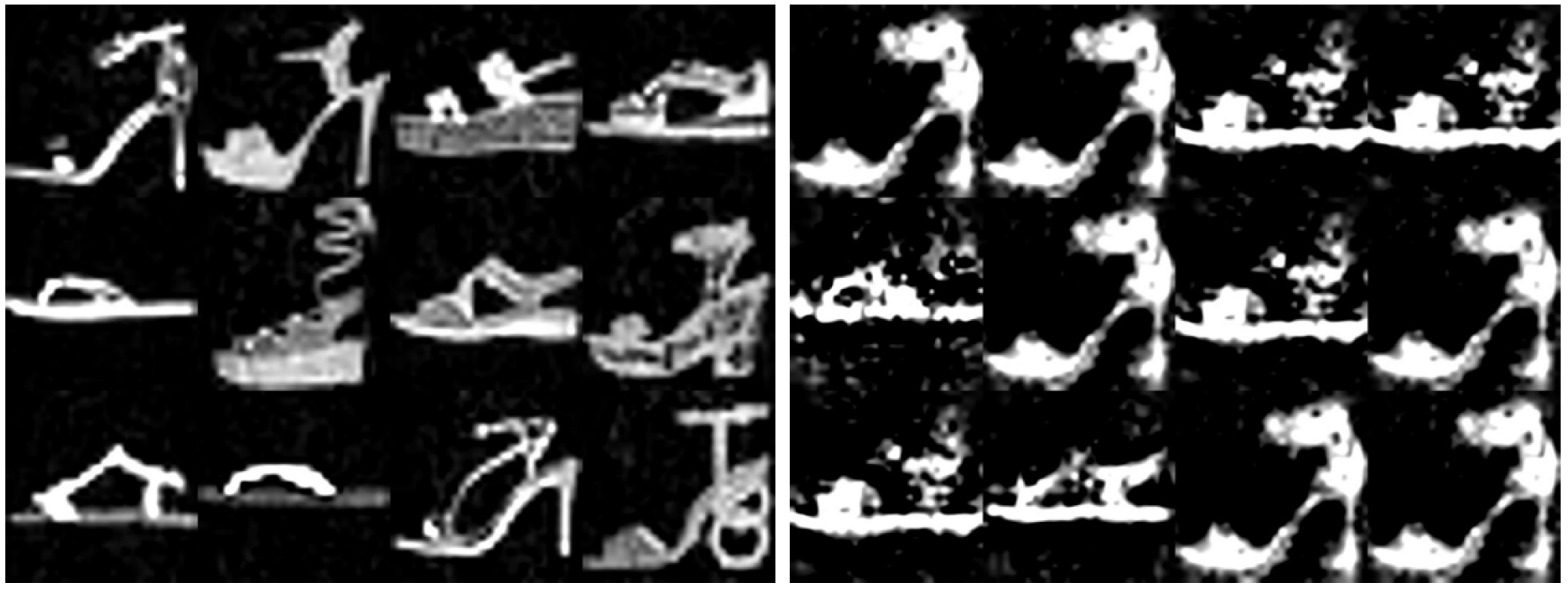}
\caption{Results for Fashion-MNIST. $N=60000$, $N_f=600$ and $N_c=300$. Left: Training data $x_i$ for class $5$. Right: $F_n(x_i)$ training data and $n=11000$. }
\label{figshoes}
\end{center}
\end{figure}

 \subsection{Sign of unsupervised Learning? }\label{secUL}
Fig.~\ref{figshoes} shoes $x_i$ and $F_n(x_i)$ for a group of images in the class 5 (sandal). The network is trained to depth $n=11000$
and the value of $\epsilon$ is chosen at each step $n$ so that the  perturbation of each data point $x_i$ of the batch
is no greater than $10\%$ ($\epsilon=0.1\times \max_i \frac{ |x^{(n)}_{f,i}|_{L^2}}{|\hat{g}^{(n)}_{f,i}|_{L^2}}$).
Note that this value of $\epsilon$ is $10$ times larger than the one of Sec.~\ref{subsec82782}.
Surprisingly the  flow $F_n$ accurately clusters that class (sandal) into $2$ sub-classes: (1)  high heels (2) flat bottom. This is  surprising because the training labels contain no information about such sub-classes:
 KF has created those clusters/sub-classes without supervision.

 \section{Kernel Flows and Convolutional Neural Network}\label{secCNN}

\subsection{MNIST} \label{KFCNN MNIST sec}
The proposed approach  can also be applied to families of kernels parameterized by the weights of a Convolutional Neural Network (CNN) \cite{lecun1995convolutional}.  Such networks are known to achieve superior performance by, to some degree, encoding (i.e. providing prior information about)  known invariants (e.g. to translations) and the hierarchical structure of data generating distribution  into the architecture of the network.

 We will first consider an application the MNIST dataset \cite{yann1998mnist} with $L^2$  normalized test and training images.
The structure of the CNN is the one presented in  \cite{CNNGorner} and its first layers are illustrated in
 Fig. \ref{MNIST layers}.
Given an input/image $x$, the last layer produces a vector  $F(x)\in \R^{300}$ used for SVN classification \cite{cortes1995support} with
the Gaussian kernel $K(x,x')=K_1(F(x), F(x')) = e^{-\gamma |F(x)-F(x')|^2}$ where $4/\gamma$ is the mean squared distance between the $F(x_i)$ (writing $x_i$ for the training images).

\begin{figure}[h]
\begin{center}
\includegraphics[width= 0.9\textwidth]{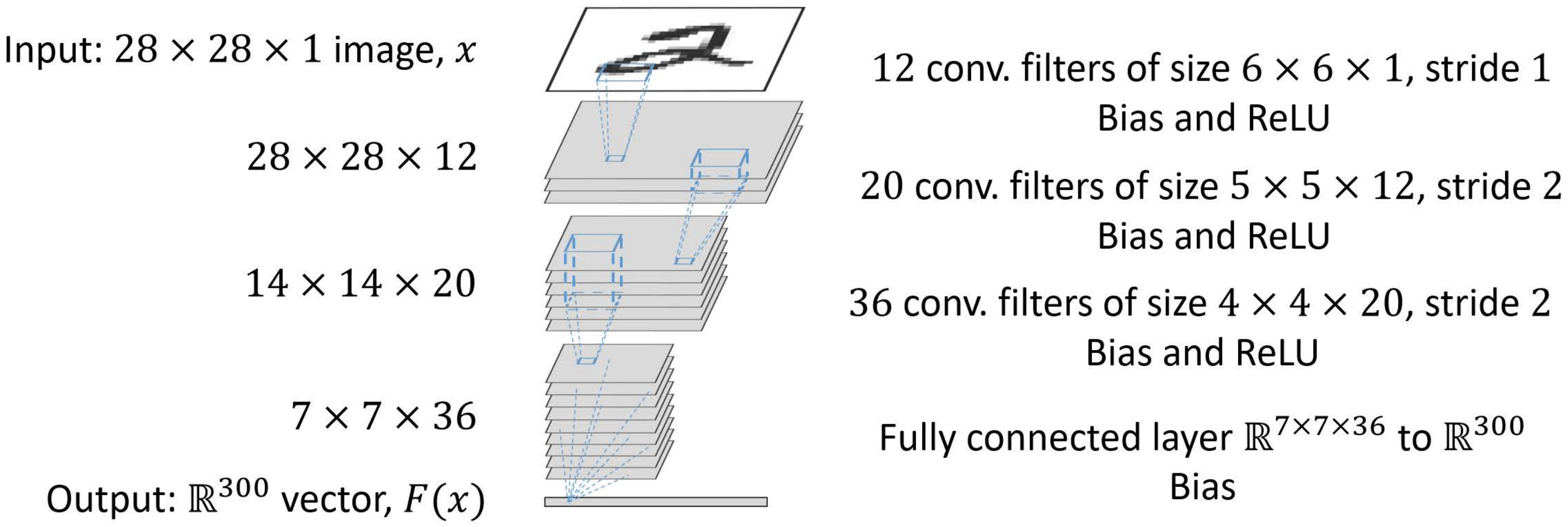}
\caption{Convolutional filters used for MNIST classification \cite{CNNGorner}.}\label{MNIST layers}

\end{center}
\end{figure}

The training of the  filters (weights of the network) is done  as described in sections \ref{sec3} and \ref{secfamker} using  random batches of $N_f = 500$ images (sampled uniformly without replacement out of $N = 60000$ training images) and  sub-batches of $N_c = 250$
images (sampled uniformly without replacement out of the batch of $N_f=500$ images). As in Sec.~\ref{sec3}, write $v^\dagger$ and $v^\s$ for optimal recoveries using the kernel $K$, and respectively, the batch of $N_f$ and the sub-batch of $N_c$ interpolation points.

Writing $y_i\in \R^{10}$ for the label of the image $x_i$, the relative approximation error (in the RKHS norm associated with $K$) caused by halving the number of points is (using, for simplicity, the notations of  Sec.~\ref{sec3} to describe the computation of $\rho$ for one batch)
\begin{equation}\label{def rho 7}
    \rho=1-\frac{\Tr{(y^T \tilde{A} y)}}{\Tr{(y^T A y)}}\,.
\end{equation}
where $y\in \R^{N_f\times 10}$, $\tilde{A}, A\in \R^{N_f\times N_f}$.
We will also consider the mean squared error
\begin{equation}\label{def e2 7}
    e_2 = \frac{2}{N_f}\sum_{i = 1}^{N_f} \big|y_{i} - v^\s (x_i)\big|^2 \,,
\end{equation}
where the sum is taken over the $N_f$ elements of the batch (note that $y_{i} - v^\s (x_i)=0$ when $i$ is in the sub-batch of $N_c$ elements used as interpolation points for $v^\s$).

To train the network we simply let the Adam optimizer \cite{AdamOpt} in TensorFlow   minimize $\rho$ or $e_2$ (used as cost functions, which does not require the manual identification of their Fr\'{e}chet derivatives with respect to the weights of the network).

Table \ref{table3} shows statistics of the corresponding test errors using the kernel $K$  learned above (using all $N=60000$ images in batches of size $N_f=500$) and five randomly selected subsets of $N_I = 12000$ training images as interpolation points.
Each run consisted of $10000$ iterations and test errors were calculated on the final iteration.   Fig. \ref{rho e plot} shows the values of $\rho$ and $e_2$ evaluated at every $100$ iterations for both algorithms (minimizing $\rho$ and $e_2$). When trained with relative entropy and dropout \cite{srivastava2014dropout}  Gorner reports \cite{CNNGorner} a minimum classification error of
$0.65\%$ testing every $100$ iterations over $5$ runs.  Since we are using the same CNN architecture, this appears to suggest that the proposed approach (of minimizing $\rho$ or $e_2$) might lead to better test accuracies than training with relative entropy and dropout.

     \begin{table}[h]
 \begin{center}
    \begin{tabular}{ | l | l | l | l | l|}
    \hline
    Algorithm & Average error & Min error & Max error & Standard Deviation \\ \hline
    Minimizing $e_2$ &     $0.596 \%$   &   $0.55 \%$   &   $0.63\%$  &   $0.032 \%$\\ \hline
    Minimizing $\rho$ &   $0.640 \%$   &   $0.60 \%$   &   $0.70\%$  &   $0.034 \%$   \\ \hline
    \end{tabular}
    \caption{Test error statistics using $N_I = 12000$ interpolation points at iteration $10000$ over $5$ runs.}\label{table3}
\end{center}
\end{table}

\begin{figure}[h]
\begin{center}
\includegraphics[width= \textwidth]{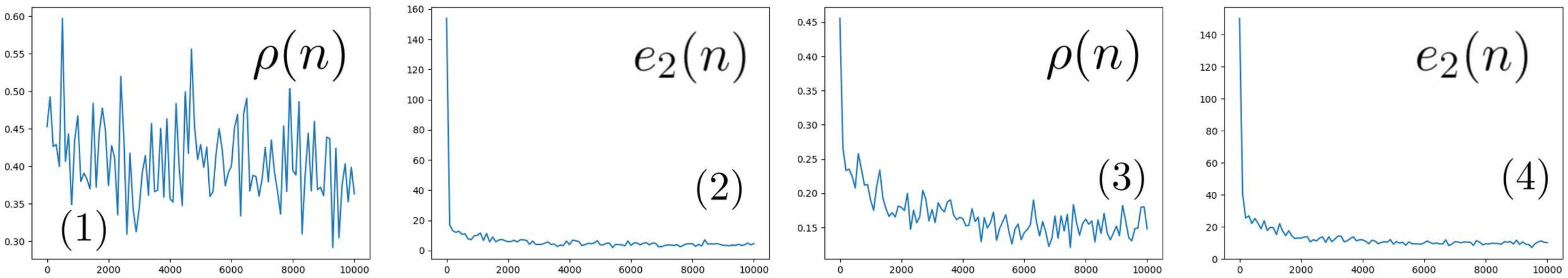}
\caption{$(1)$ and $(2)$ show $\rho$ and $e_2$ respectively evaluated at the $n$-th batch using the $e_2$ minimizing network.  $(3)$ and $(4)$ show analogous plots for the $\rho$ minimizing network.}
\label{rho e plot}
\end{center}
\end{figure}

\subsubsection{Interpolation with small subsets of the training set}
Fig.~\ref{multi_erho} shows test errors using the kernel $K$  learned above (using all $N=60000$ images in batches of size $N_f=500$) and randomly selected subsets of $N_I = 30000, 12000, 6000, 600, 60, 10$ training images as interpolation points.

\begin{figure}[h!]
\begin{center}
\includegraphics[width= \textwidth]{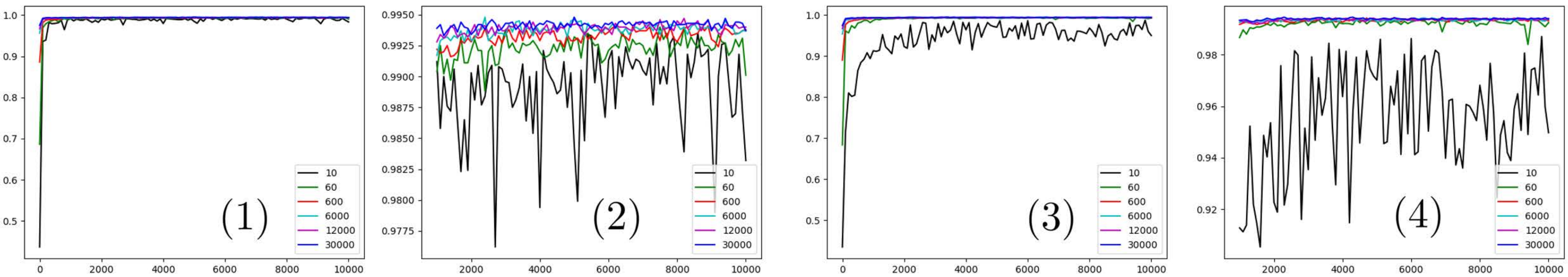}
\caption{(1) Classification test errors for $N_I = 10, 60, 600, 6000, 12000, 30000$  evaluated at the $n$-th batch for $0\leq n \leq 10000$ using the network minimizing $e_2$.  (2) same as (1) with $1000 \leq n \leq 10000$.  (3), (4) same as (1), (2) for the network minimizing $\rho$.}
\label{multi_erho}
\end{center}
\end{figure}

Tables  \ref{table5} and \ref{table7} show test errors statistics using the kernel $K$ (learned above with $N_f=500$) with
$N_I=6000, 600, 60, 10$ interpolation points sampled at random (all  use the same convolutional filters obtained in a single optimization run). Averages, min, max and STD are computed over iterations between iterations $9900$ to $10000$
 using $5$ independent runs of the Adam optimizer \cite{AdamOpt} with
$\rho$ and $e_2$ as objective functions.

     \begin{table}[h!]
 \begin{center}
    \begin{tabular}{ | l | l | l | l | l|}
    \hline
    $N_I$ & Average error & Min error & Max error & Standard Deviation \\ \hline
    $6000$ &     $0.575 \%$   &   $0.42 \%$   &   $0.72\%$  &   $0.052 \%$\\ \hline
    $600$ &   $0.628 \%$   &   $0.48 \%$   &   $0.83\%$  &   $0.062 \%$   \\ \hline
    $60$ & $0.728 \%$   &   $0.51 \%$   &   $1.23\%$  &   $0.103 \%$  \\
    \hline
    $10$ &   $1.05 \%$   &   $0.58 \%$   &   $4.81\%$  &   $0.375 \%$\\ \hline
    \end{tabular}
    \caption{Test error statistics using $N_I$ interpolation points between iterations $9900$ and $10000$ over $5$ runs of optimizing $e_2$.}\label{table5}
\end{center}
\end{table}

     \begin{table}[h!]
 \begin{center}
    \begin{tabular}{ | l | l | l | l | l|}
        \hline
    $N_I$ & Average error & Min error & Max error & Standard Deviation \\ \hline
    $6000$ &     $0.646 \%$   &   $0.51 \%$   &   $0.78\%$  &   $0.046 \%$\\ \hline
    $600$ &   $0.676 \%$   &   $0.56 \%$   &   $0.82\%$  &   $0.047 \%$   \\ \hline
    $60$ & $0.850 \%$   &   $0.58 \%$   &   $3.98\%$  &   $0.357 \%$  \\
    \hline
    $10$ &   $4.434 \%$   &   $0.97 \%$   &   $18.91\%$  &   $2.320 \%$\\ \hline
    \end{tabular}
    \caption{Test error statistics using $N_I$ interpolation points between iterations $9900$ and $10000$ over $5$ runs of optimizing $\rho$.}\label{table7}
\end{center}
\end{table}

\begin{figure}[h!]
\begin{center}
\includegraphics[width= 0.6\textwidth]{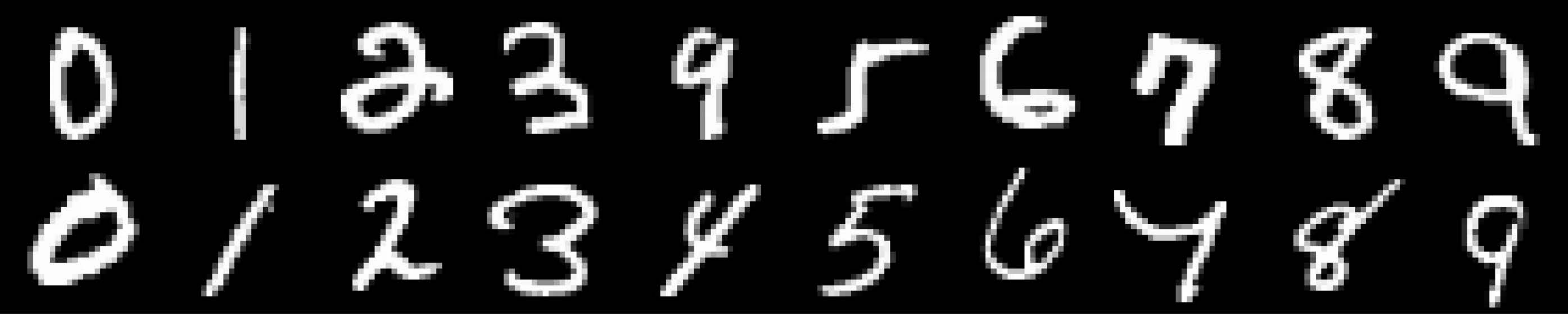}
\caption{A ``bad'' (top) and ``good'' (bottom) selection of $10$ interpolation points.}
\label{badgood10}
\end{center}
\end{figure}

Observe that, although as with Kernel Flow, using only a small fraction of the  training data as interpolation points is
sufficient to achieve low classification errors (the minimum error with $10$ interpolation points is $0.58\%$), interpolation with only one image per class appears to be more sensitive to the particular selection of $10$ interpolation points. Fig.~\ref{badgood10} shows an example of a  ``good'' and a ``bad'' selection for the interpolation with $10$ points.

The clustering of the $F(x_i)$ ($x_i$ are training images and $F(x)\in \R^{300}$ is the output of the last layer of the CNN with input $x$) is a possible explanation for this extreme  generalization.  Fig.~\ref{dist plot} shows the average mean squared Euclidean distance between
$F(x_i)$ in the same class and in distinct classes.
Note that the ratio between average square distances between two inter-class and two in-class points approaches $12$ (for the network optimizing $e_2$), suggesting that the map $F$  clusters of images per class.

\begin{figure}[h!]
\begin{center}
\includegraphics[width= \textwidth]{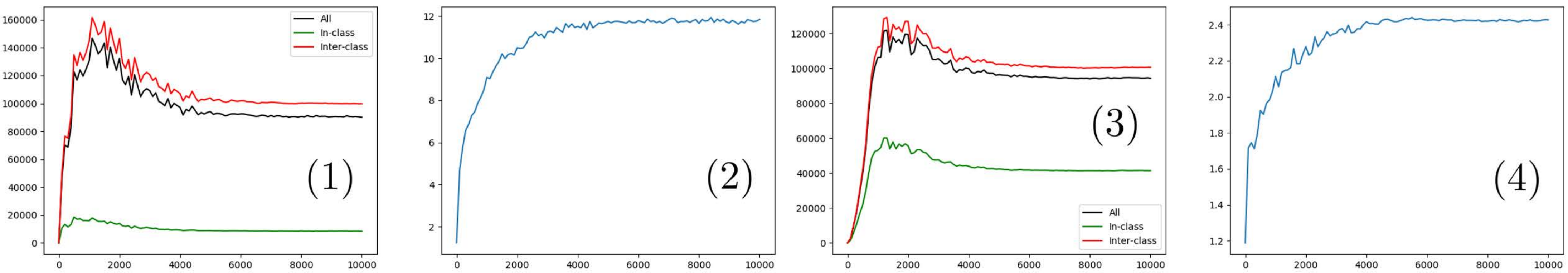}
\caption{$(1)$ Mean-squared distance between $F(x_i)$ (all, in-class, and inter-class) vs iteration $n$ for the network optimizing $e_2$ $(2)$ Ratio between inter-class and in-class mean-squared distance for the network optimizing $e_2$.  $(3)$ and $(4)$ are identical except for the network which optimizes $\rho$.}
\label{dist plot}
\end{center}
\end{figure}

\subsection{Fashion MNIST}

We now apply the proposed approach to  the Fashion-MNIST database.  The architecture of the CNN is derived from \cite{FashionCNN} and the first layers of the network are shown in Fig. \ref{Fashion MNIST layers}.

\begin{figure}[h]
\begin{center}
\includegraphics[width= 0.9\textwidth]{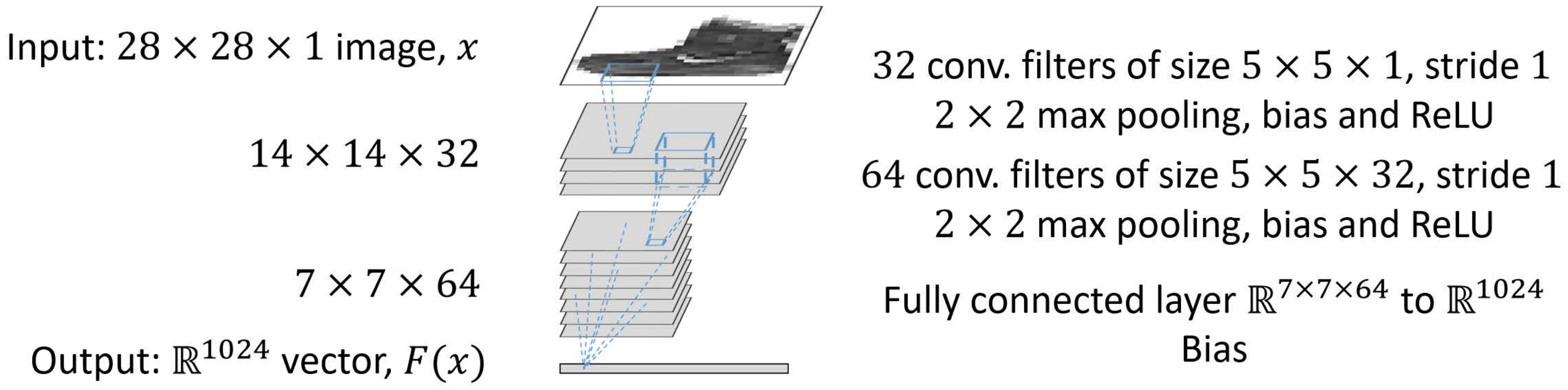}
\caption{Convolutional filters used for Fashion MNIST classification \cite{FashionCNN}.}\label{Fashion MNIST layers}

\end{center}
\end{figure}

The classification of test images is done as in Sec.~\ref{KFCNN MNIST sec}.

Table \ref{rho e error fashion} shows  test errors statistics (over $5$ runs) after training (using $10000$ iterations and computing test errors at the final iteration) by minimizing  $\rho$ or $e_2$
 (as defined in \eqref{def rho 7} and \eqref{def e2 7}) using  $N_I = 12000$ interpolation points
  Mahajan \cite{FashionCNN} reports a testing error of $8.6\%$ when using the validation set to obtain the convolutional filters.  As above, this suggests that the proposed approach could lead to better test accuracies than training with relative entropy and dropout.  Finally, Fig.~\ref{rho e plot fashion} shows $\rho$ and $e_2$ evaluated at every $100$ iterations for both algorithms.

     \begin{table}[h]
 \begin{center}
    \begin{tabular}{ | l | l | l | l | l|}
    \hline
    Algorithm & Average error & Min error & Max error & Standard Deviation \\ \hline
    Optimizing $e_2$ &     $8.474 \%$   &   $8.24 \%$   &   $8.70\%$  &   $0.147 \%$\\ \hline
    Optimizing $\rho$ &   $8.412 \%$   &   $8.29 \%$   &   $8.56\%$  &   $0.091 \%$   \\ \hline
    \end{tabular}
    \caption{Test error statistics using $N_I = 12000$ interpolation points at iteration $10000$ over $5$ runs.}\label{rho e error fashion}
\end{center}
\end{table}

\begin{figure}[h]
\begin{center}
\includegraphics[width= \textwidth]{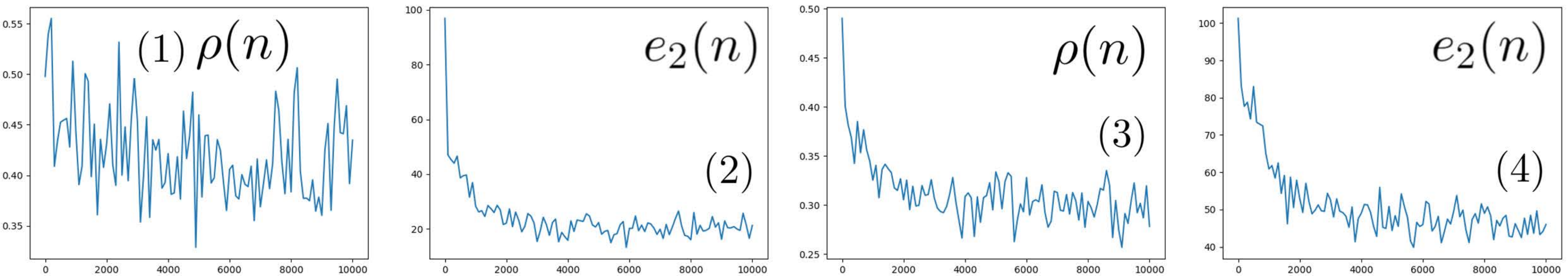}
\caption{$(1)$ and $(2)$ show $\rho$ and $e_2$ respectively evaluated at the $n$-th batch using the $e_2$ minimizing network.  $(3)$ and $(4)$ show analogous plots for the $\rho$ minimizing network.}
\label{rho e plot fashion}
\end{center}
\end{figure}

\subsubsection{Interpolation with small subsets of the training set}
Fig.~\ref{multi_erho fashion} shows test errors using the kernel $K$  learned above (using all $N=60000$ images in batches of size $N_f=500$) and randomly selected subsets of $N_I = 30000, 12000, 6000, 600, 60, 10$ training images as interpolation points.

Further, the minimum errors in Fig.~\ref{multi_erho fashion}.1, 3 are observed to be $8.02 \%$ and $7.89 \%$ respectively, where both minima used $N_I = 30000$."

\begin{figure}[h!]
\begin{center}
\includegraphics[width= \textwidth]{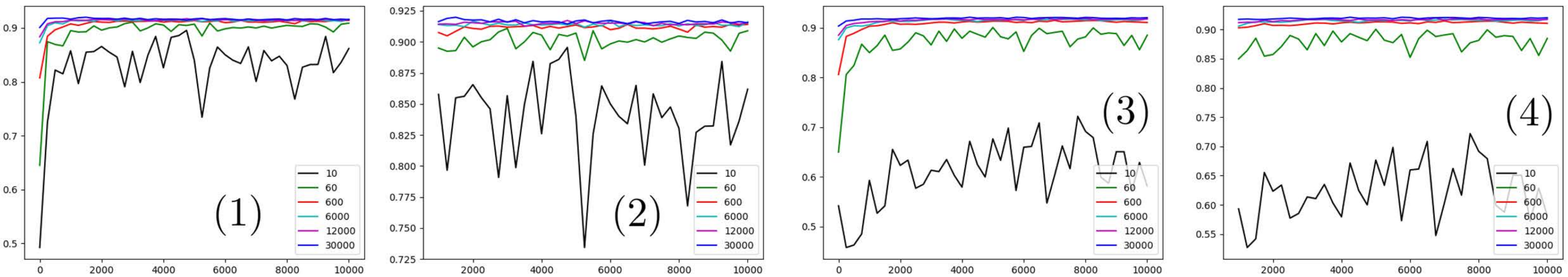}
\caption{(1) Classification test errors for $N_I = 10, 60, 600, 6000, 12000, 30000$  evaluated at the $n$-th batch for $0\leq n \leq 10000$ using the network minimizing $e_2$.  (2) same as (1) with $1000 \leq n \leq 10000$.  (3), (4) same as (1), (2) for the network minimizing $\rho$.}
\label{multi_erho fashion}
\end{center}
\end{figure}

Tables \ref{10k e2 fashion} and \ref{10k rho fashion} show test errors statistics using the kernel $K$ (learned above with $N_f=500$) with
$N_I=6000, 600, 60, 10$ interpolation points sampled at random (all  use the same convolutional filters obtained in a single optimization run). Averages, min, max and STD are computed over iterations between $9900$ and $10000$
 using $5$ independent runs of the Adam optimizer \cite{AdamOpt} with
$\rho$ and $e_2$ as objective functions.

     \begin{table}[h!]
 \begin{center}
    \begin{tabular}{ | l | l | l | l | l|}
    \hline
    $N_I$ & Average error & Min error & Max error & Standard Deviation \\ \hline
    $6000$ &     $8.561 \%$   &   $8.23 \%$   &   $8.97\%$  &   $0.135 \%$\\ \hline
    $600$ &   $8.724 \%$   &   $8.31 \%$   &   $9.26\%$  &   $0.161 \%$   \\ \hline
    $60$ & $9.677 \%$   &   $8.77 \%$   &   $11.48\%$  &   $0.486 \%$  \\
    \hline
    $10$ &   $15.261 \%$   &   $10.00 \%$   &   $32.69\%$  &   $3.196 \%$\\ \hline
    \end{tabular}
    \caption{Test error statistics using $N_I$ interpolation points between iterations $9900$ and $10000$ over $5$ runs of optimizing $e_2$.}\label{10k e2 fashion}
\end{center}
\end{table}

     \begin{table}[h!]
 \begin{center}
    \begin{tabular}{ | l | l | l | l | l|}
    \hline
    $N_I$ & Average error & Min error & Max error & Standard Deviation \\ \hline
    $6000$ &     $8.526 \%$   &   $8.17 \%$   &   $8.96\%$  &   $0.120 \%$\\ \hline
    $600$ &   $8.810 \%$   &   $8.36 \%$   &   $9.29\%$  &   $0.140 \%$   \\ \hline
    $60$ & $11.677 \%$   &   $9.32 \%$   &   $18.03\%$  &   $1.437 \%$  \\
    \hline
    $10$ &   $36.642 \%$   &   $23.44 \%$   &   $53.56\%$  &   $4.900 \%$\\ \hline
    \end{tabular}
    \caption{Test error statistics using $N_I$ interpolation points between iterations $9900$ and $10000$ over $5$ runs of optimizing $\rho$.}\label{10k rho fashion}
\end{center}
\end{table}

It can again be observed that using only a small fraction of the training data as interpolation points yields relatively low classification errors. The instability of test errors with $N_I = 10$ interpolation points, compared to the Kernel Flow algorithm proposed in Sec.~\ref{secKernelFlow}  seem to suggest that deep architectures might be required to achieve stable results with only one interpolation point per  class.

The clustering of the $F(x_i)$ ($x_i$ are training images and $F(x)\in \R^{300}$ is the output of the last layer of the CNN with input $x$) using the network optimizing $e_2$ is shown in Fig.~\ref{dist plot fashion}.  The figure shows the average mean squared Euclidean distance between $F(x_i)$ in the same class and in distinct classes.  Note that the ratio between average square distances between two inter-class and two in-class points approaches $4.5$, suggesting that the map $F$ aggregates images per class.

\begin{figure}[h!]
\begin{center}
\includegraphics[width= \textwidth]{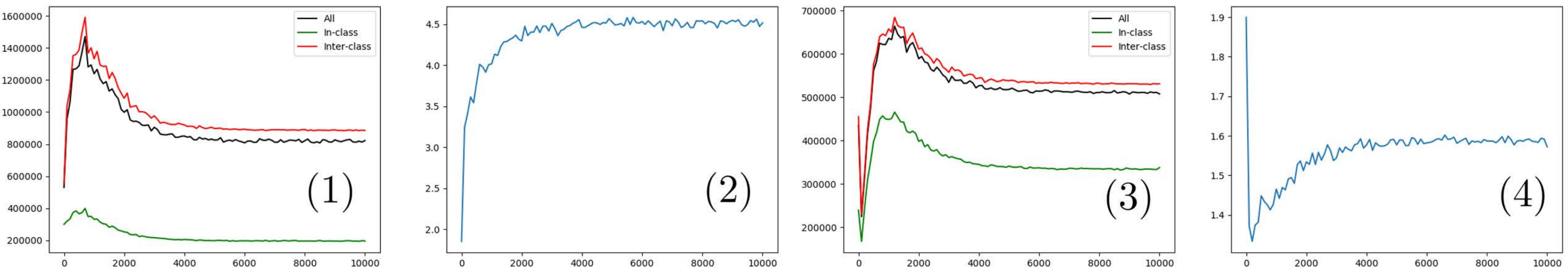}
\caption{$(1)$ Mean-squared distance between $F(x_i)$ (all, in-class, and inter-class) vs iteration $n$ for the network optimizing $e_2$ $(2)$ Ratio between inter-class and in-class mean-squared distance for the network optimizing $e_2$.  $(3)$ and $(4)$ are identical except for the network which optimizes $\rho$.}
\label{dist plot fashion}
\end{center}
\end{figure}

\paragraph{Acknowledgments.}
The authors gratefully acknowledges this work supported by  the Air Force Office of Scientific Research and the DARPA EQUiPS Program under award   number FA9550-16-1-0054 (Computational Information Games) and the Air Force Office of Scientific Research under award number FA9550-18-1-0271 (Games for Computation and Learning).

\bibliographystyle{plain}
\bibliography{merged,RPS,extra}

\end{document}